\documentclass[letterpaper]{article} 
\usepackage[]{aaai24}  
\usepackage{graphicx}
\usepackage{epstopdf}

\usepackage{times}  
\usepackage{helvet}  
\usepackage{courier}  
\usepackage[hyphens]{url}  
\usepackage{graphicx} 
\usepackage{natbib}  
\usepackage{caption} 
\newif\ifshowappendix
\showappendixtrue 
\usepackage{xcolor}

\usepackage{algorithm}
\usepackage{algorithmic}

\usepackage{lineno}

\usepackage{amsfonts, amsmath, amsthm, amssymb}
\usepackage{microtype}
\usepackage{multirow}
\usepackage{diagbox}
\usepackage{graphicx}
\usepackage{subfigure}
\usepackage{booktabs} 
\usepackage{algorithm}
\usepackage{dsfont}

\newtheorem{theorem}{Theorem}[]

\newtheorem{corollary}{Corollary}[]

\newtheorem{lemma}{Lemma}[]
\usepackage{newfloat}
\usepackage{listings}
\DeclareCaptionStyle{ruled}{labelfont=normalfont,labelsep=colon,strut=off} 
\floatstyle{ruled}
\newfloat{listing}{tb}{lst}{}
\floatname{listing}{Listing}
\title{On Estimating the Gradient of the Expected Information Gain 
in Bayesian Experimental Design}
\author{
    Ziqiao Ao,~Jinglai Li
}
\affiliations{
    School of Mathematics, University of Birmingham\\

    \{zxa029, j.li.10\}@bham.ac.uk
}

\usepackage{bibentry}

\begin{document}
\maketitle

\begin{abstract}
 Bayesian Experimental Design (BED), 
 which aims to find the optimal experimental conditions for Bayesian inference,
is usually posed as
to optimize the expected information gain (EIG). 
The gradient information is often needed 
  for efficient EIG optimization,
and as a result the ability to estimate 
the gradient of EIG is essential for BED problems. 
The primary goal 
of this work is to develop methods 
for estimating the gradient of EIG,
which, combined with the stochastic gradient descent algorithms,  
result in efficient optimization of EIG. 
Specifically, we first introduce a posterior expected representation of the EIG gradient with respect to the design variables. Based on this, we propose two  methods for estimating the EIG gradient, UEEG-MCMC that leverages posterior samples generated through Markov Chain Monte Carlo (MCMC) to estimate the EIG gradient, and BEEG-AP that focuses on achieving high simulation efficiency by repeatedly using parameter samples. Theoretical analysis and numerical studies illustrate that UEEG-MCMC 
is robust agains the actual EIG value, while BEEG-AP is more efficient when the EIG value to be optimized is small. Moreover, both methods show superior performance compared to several popular benchmarks in our numerical experiments.
\end{abstract}

\section{Introduction}
The advancement of science and engineering heavily relies on the acquisition of data through experiments. However, conducting experiments can be resource-intensive and time-consuming. To maximize the information gained from collected data and optimize experimental outcomes, researchers turn to Bayesian experimental design (BED) which offers a systematic and powerful framework for making informed decisions about experimental setups and selecting optimal conditions for data collection.
At its core, BED aims to strategically allocate resources to collect the most informative data, which leads to more accurate parameter estimation, model validation, and decision-making. It has been broadly applied in diverse scientific fields, including pharmacokinetic study \cite{ryan2014towards}, drug discovery \cite{lyu2019ultra}, systems biology \cite{kreutz2009systems}, compressed sensing \cite{seeger2008compressed} and physics simulations \cite{melendez2021designing}.

Mathematically, BED can be formulated as an optimization problem, where the objective is to maximize a specific function known as the utility function. The choice of the utility function is often driven by the purpose of the BED. In this paper, our focus lies on BED's application to precise parameter estimation. Various utility functions have been employed to address parameter estimation challenges, and some notable examples include Bayesian A-posterior precision, Bayesian D-posterior precision, quadratic loss and expected information gain (EIG) (see \cite{ryan2016review} for a review). While the EIG stands out for its exceptional theoretical appeal among these functions, its use has been a long-standing challenge historically due to the computational complexity associated with estimating the evidence or marginal likelihood, which is analytically intractable or computationally expensive to evaluate directly. Consequently, BED based on EIG was once restricted to special cases where analytical solutions or simplifying assumptions allowed for a tractable computation of the evidence \cite{borth1975total, lewi2009sequential}. Fortunately,  recent advancements in Artificial Intelligence (AI) tools, particularly in neural EIG estimators and automatic differentiation frameworks, have significantly alleviated the computational challenge, enabling the efficient implementation of EIG-based BED in a wide range of non-linear and high-dimensional problems.

Accurate estimation of EIG has been widely recognized as one of the most significant barriers of EIG-based BED. However, our primary interest is not in the exact value of the EIG, but rather in the design variables that maximize the value. Motivated by this viewpoint, an alternative strategy is directly estimating the gradient of EIG w.r.t. the design variables and then using stochastic gradient descent to search for the optimal design. In this paper, we propose two methods for estimating the EIG gradient. The first method, UEEG-MCMC, applies posterior samples generated by Markov Chain Monte Carlo (MCMC) to estimate the EIG gradient. It is shown effective across different scenarios, regardless of the ground-truth EIG values. The second method, BEEG-AP, is more simulation-efficient. However, its performance suffers when dealing with problems that have large ground-truth EIG values. The paper establishes a connection with nested Monte Carlo to analyze this behavior, shedding light on the limitations of BEEG-AP in such cases.

We validate the aforementioned attributes of the two proposed methods through a meticulous numerical experiment and diverse applications featuring varying expected EIG levels. Additionally, comprehensive comparisons are made with several bench-marking approaches, revealing the superior performance of our proposed methods.


\subsection{Related Work}
Rainforth et al.\cite{rainforth2023modern} provide a thorough review of modern methods for Bayesian experimental design. Early schemes for Bayesian experimental design used separate stages to estimate the Expected Information Gain (EIG) and optimize the design variables $\lambda$, both of which can be challenging tasks. For EIG estimation, the main computational challenge arises from the intractability of $p(y|\lambda)$ and several methods have been proposed to solve this problem. Notably,  Nested Monte Carlo (NMC) \cite{rainforth2018nesting, ryan2003estimating} emerged as a prominent method in this area.  Additionally, Variational EIG estimators (also known as variational mutual information estimators) \cite{foster2019variational, barber2004algorithm, donsker1975asymptotic,nguyen2010estimating} combined with deep learning techniques \cite{belghazi2018mutual,oord2018representation,alemi2016deep,song2019understanding} showed significant progress. Furthermore, alternative approaches for EIG estimation included using ratio estimation \cite{thomas2022likelihood} as proposed in \cite{kleinegesse2019efficient}, and bounding EIG from below by two or more entropies in the data space \cite{ao2020approximate} which are then be estimated by entropy estimation methods \cite{kraskov2004estimating}. A direct lower bound estimation for EIG was introduced in \cite{tsilifis2017efficient} for models with fixed normally distributed measurement noises, and the EIG Laplace approximation was proposed in \cite{long2013fast}. Regarding the optimization of $\lambda$, conventional gradient-free approaches such as Bayesian optimization (BO) \cite{snoek2012practical}, Simultaneous perturbation stochastic approximation (SPSA) \cite{spall1998overview}, simulation-based optimization (SBO)\cite{muller2005simulation}, and Nelder-Mead nonlinear simplex (NMNS)\cite{nelder1965simplex} were commonly employed. However, these methods encountered scalability issues when dealing with high-dimensional design variable spaces. 

Recent advancements have introduced efficient gradient-based approaches that leverage the reparameterization trick and automatic differentiation frameworks.  These methods, such as those proposed in \cite{foster2020unified, kleinegesse2020bayesian,zhang2021scalable,zaballa2023stochastic}, allow for simultaneous optimization of both the variational parameters and the design variables. Moreover, Goda et al. \cite{goda2022unbiased} presented a method that directly obtains an unbiased estimator of the EIG gradient using a randomized
version of multilevel Monte Carlo (MLMC) method \cite{rhee2015unbiased}. Furthermore, gradient estimators for implicit models \cite{li2017gradient, wen2021gradient, lim2020ar, shi2018spectral} with score matching techniques \cite{hyvarinen2005estimation, song2020sliced} have emerged as another avenue for optimizing $\lambda$ in a gradient-based way. Despite receiving limited attention in the Bayesian experimental design community, these methods hold promise and deserve further exploration.

\section{Preliminary Knowledge}
\label{maxentsec:pk}
\subsection{Problem Formulation}
This section defines the variables and functions involved as well as the primary objective of Bayesian experimental design. Let $\lambda\in \mathcal{D}$ be the design variables that can be controlled by users. The parameters to be inferred from the observed data $y$ are denoted by $\theta$ and $\pi_\mathrm{\theta}(\theta)$ denotes its prior that represents our knowledge or belief about the parameters before observing any data. $\epsilon$ represents the base model noises generated from a known distribution $\pi_\mathrm{\epsilon}(\epsilon)$. The process of simulating the observed data can be modeled by a sampling path
\begin{equation}
    y = g(\theta, \epsilon, \lambda).
\end{equation}
In this context, we assume the existence of a tractable likelihood function $l(y| \theta,\lambda)$, which is derived from the sampling path. However, it is important to note that the marginal likelihood $p(y|\lambda)$ is not available in closed-form. Under certain design $\lambda$,
the expected information gain (EIG) is defined as
    \begin{equation}\label{maxent:eq1}
        U(\lambda) =\mathbb{E}_{\pi_{\mathrm{\theta}(\theta)l(y|\theta, \lambda)}}[\log l(y|\theta, \lambda)]-\mathbb{E}_{p(y|\lambda)}[\log p(y|\lambda)].
    \end{equation}
It represents the expected amount of information that observations $y$ provide about $\theta$ and is well known as the mutual information \cite{cover1999elements} between $\theta$ and $y$ in information theory community. The objective of Bayesian experimental design is then to maximize the EIG over the design variable space $\mathcal{D}$
\begin{equation}
    \lambda^* = \underset{\lambda \in \mathcal{D}}{\arg\max }~U(\lambda).
\end{equation}

\subsection{Simulation Cost}
Estimating and optimizing the EIG entails generating observation samples and evaluating the corresponding likelihood values. This process can be computationally demanding, and its efficiency is of paramount importance in practical applications. In this section, we will illustrate how to quantify the simulation cost during the process when explicit models are considered.
For explicit models, the sampling path can typically be written as a hierarchical structure
\begin{equation}
    y = g(f(\theta, \lambda), \epsilon, \theta, \lambda),
\end{equation}
where $f$ is the forward model which dominates the main computational cost. Examples of the commonly used hierarchical structure include:
\begin{equation}
    \mathrm{Additive~noise:~} y = f(\theta, \lambda)+\sigma(\theta,\lambda)\epsilon.
\end{equation}
\begin{equation}
    \mathrm{Multiplicative~noise:~} y = f(\theta, \lambda)(1+\sigma(\theta,\lambda)\epsilon).
\end{equation}
\begin{equation}
    \mathrm{Mixture~of~noises:~} y \!=\! f(\theta\!, \lambda)(1+\sigma_1(\theta\!,\lambda)\epsilon_1)+\sigma_2(\theta\!,\lambda)\epsilon_2.
\end{equation}

For simplicity, we take the case of additive noise for example to analyze the simulation cost concerned in applications. Given a fixed parameter $\theta^*$, simulating multiple observations (e.g. $y^{(k)}=f(\theta^*, \lambda)+\sigma(\theta^*,\lambda)\epsilon^{(k)}$, $k=1,...,K$) only involves a single simulation (i.e. $f(\theta^*, \lambda)$) of the forward model. Likewise, since the likelihood function can be analytically written as 
\begin{equation}
    l(y| \theta,\lambda) = \pi_\mathrm{\epsilon}\Big(\frac{y-f(\theta,\lambda)}{\sigma(\theta,\lambda)}\Big),
\end{equation}
evaluating the values of the likelihood w.r.t. multiple observations (e.g. $l(y^{(k)}| \theta^*,\lambda)$, $k=1,...,K$) also requires only a single forward pass of $f$. Thus, the simulation cost of generating observations from and evaluating the likelihood is directly related to the number of different parameters involved. This observation is important for the analysis of simulation costs of estimators in the later sections.

\section{Posterior Expected Representations of the EIG Gradient}\label{maxentsec:reig}
In this section, we introduce a novel representation of the EIG gradient that offers new insights into the development of efficient gradient estimators. To start with, we analyze the difficulty in directly computing the EIG gradient w.r.t. the design variables $\lambda$. Estimating the gradient of Eq.~\eqref{maxent:eq1} directly with score-function estimators \cite{mohamed2020monte} could lead to high variance. As a result,  practitioners often turn to pathwise gradient estimators (also known as the reparameterization tricks) \cite{mohamed2020monte}, as an alternative strategy for estimating this gradient:  
\begin{equation}
    \begin{aligned}
        \nabla_\lambda U(\lambda)
        =&\nabla_\lambda \mathbb{E}_{\pi_{\mathrm{\theta}(\theta)\pi_{\epsilon}(\epsilon)}}[\log l(g(\theta,\epsilon,\lambda)|\theta, \lambda)]
        \\& \quad\quad\quad\quad
        -\nabla_\lambda \mathbb{E}_{\pi_{\mathrm{\theta}(\theta)\pi_{\epsilon}(\epsilon)}}[\log p(g(\theta,\epsilon,\lambda)|\lambda)]\\
        =& \mathbb{E}_{\pi_{\mathrm{\theta}(\theta)\pi_{\epsilon}(\epsilon)}}[\nabla_\lambda\log l(g(\theta,\epsilon,\lambda)|\theta, \lambda)]
        \\ & \quad\quad\quad\quad
        - \mathbb{E}_{\pi_{\mathrm{\theta}(\theta)\pi_{\epsilon}(\epsilon)}}[\nabla_\lambda\log p(g(\theta,\epsilon,\lambda)|\lambda)].
    \end{aligned}
\end{equation}

While  obtaining  $\nabla_\lambda g(\theta,\epsilon,\lambda)$ is typically straightforward using modern automatic differentiation
frameworks (e.g. Tensorflow \cite{abadi2016tensorflow} and Pytorch \cite{paszke2017automatic}), the score function $\nabla_y \log p(y|\lambda)$ usually does not have an analytical form, rendering the second term of the above estimator intractable. To address this challenge, we apply the key idea in \cite{brehmer2020mining}, which involves using the tractable score of likelihood $\nabla_y \log l(y|\theta,\lambda)$ to estimate the intractable score function $\nabla_y \log p(y|\lambda)$. This can be summarized as the following Lemma~\ref{maxent:lemma1}.

\begin{lemma}\label{maxent:lemma1}
    The gradient of the logarithm of the marginal density w.r.t. the experimental condition $\lambda$ admits the following representation:
    \begin{equation}
        \begin{aligned}
            &\nabla_\lambda \log p(g(\theta,\epsilon,\lambda))|\lambda)\\
            &\quad\quad\quad\quad 
            =
            -\mathbb{E}_{q(\theta'|g(\theta,\epsilon,\lambda),\lambda)}[\nabla_\lambda \log l(g(\theta,\epsilon,\lambda)|\theta', \lambda)],
        \end{aligned}
    \end{equation}
    where $q(\theta'|y,\lambda)\propto \pi_{\mathrm{\theta}}(\theta')l(y|\theta',\lambda)$ is the posterior density of parameters given the observation sample $y$.
\end{lemma}

Using this lemma, we derive an entropy gradient estimator for the marginal distribution of $y$ as stated in Theorem~\ref{maxent:thm1}.

\begin{theorem}\label{maxent:thm1}
    The gradient of the entropy $H(p(y| \lambda))$ w.r.t. the experimental condition $\lambda$ satisfies
    \begin{equation}
        \begin{aligned}
            &\nabla_\lambda H(p(y|\lambda)) \\
            =& -\mathbb{E}_{\pi_\mathrm{\theta}(\theta)\pi_\mathrm{\epsilon}(\epsilon) q(\theta'|g(\theta,\epsilon,\lambda),\lambda)}[\nabla_\lambda\log l(g(\theta,\epsilon,\lambda)|\theta',\lambda)],
        \end{aligned}
    \end{equation}
    where $q(\theta'|y,\lambda)\propto \pi_{\mathrm{\theta}}(\theta')l(y|\theta',\lambda)$ is the posterior density of parameters given the observation sample $y$.
\end{theorem}


Using Theorem~\ref{maxent:thm1}, we can get a posterior expected representation of the EIG gradient, as stated in the following Corollary~\ref{matent:corollary1}.

\begin{corollary}\label{matent:corollary1}
    The gradient of the EIG $U(\lambda)$ w.r.t. the experimental condition $\lambda$ satisfies
    \begin{equation}\label{maxent:eq2}
        \begin{aligned}
            &\nabla_\lambda U(\lambda)\\
            & \quad=
        \mathbb{E}_{\pi_\mathrm{\theta}(\theta)\pi_\mathrm{\epsilon}(\epsilon) q(\theta'|g(\theta,\epsilon,\lambda),\lambda)}[\nabla_\lambda\log l(g(\theta,\epsilon,\lambda)|\theta,\lambda)
        \\& \quad\quad\quad\quad\quad\quad\quad\quad\quad\quad\quad\quad
        -\nabla_\lambda\log l(g(\theta,\epsilon,\lambda)|\theta',\lambda)],
        \end{aligned}
    \end{equation}
    where $q(\theta'|y,\lambda)\propto \pi_{\mathrm{\theta}}(\theta')l(y|\theta',\lambda)$ is the posterior density of parameters given the observation sample $y$.
\end{corollary}

\section{Estimating the EIG Gradient}\label{maxentsec:eeig}
Building upon the posterior expected representation of the EIG gradient in Eq.~\eqref{maxent:eq2}, we propose two estimators of EIG gradient. When integrated with stochastic gradient descent algorithms, these estimators seamlessly evolve into the respective algorithms for Bayesian experimental design. For simplicity, we denote the observation samples generated from the sampling path as $y^{(i)}(\lambda) = g(\theta^{(i)},\epsilon^{(i)},\lambda)$ throughout this section.

\subsection{Unbiased Estimation of EIG Gradient with Markov Chain Monte Carlo}
The most straightforward method to estimate the expectation in Eq.~\eqref{maxent:eq2} is utilizing MCMC schemes. Specifically, giving samples $\{\theta^{(i)}\}_{i=1}^M$ and $\{\epsilon^{(i)}\}_{i=1}^M$ drawn from $\pi_\mathrm{\theta}(\theta)\pi_\mathrm{\epsilon}(\epsilon)$, the expectation in Eq.~\eqref{maxent:eq2} can be estimated by Monte Carlo average as 
\begin{equation}\label{maxent:eq3}
    \begin{aligned}
        \nabla_\lambda U(\lambda) &\approx 
    \frac{1}{M}\sum_{i=1}^M\nabla_\lambda\log l(y^{(i)}(\lambda)|\theta^{(i)},\lambda)
    \\&
    -\frac{1}{M}\sum_{i=1}^M \mathbb{E}_{q(\theta'|y^{(i)}(\lambda),\lambda)}[\nabla_\lambda\log l(y^{(i)}(\lambda)|\theta',\lambda)].
    \end{aligned}
\end{equation}
Ideally if we can draw samples $\{\theta^{(i,j)}\}_{j=1}^N$ exactly from the posterior $q(\theta'|y^{(i)}(\lambda),\lambda)$, 
we can obtain an unbiased estimator of 
$\mathbb{E}_{q(\theta'|y^{(i)}(\lambda),\lambda)}[\nabla_\lambda\log l(y^{(i)}(\lambda)|\theta',\lambda)] $:
\begin{equation}\label{maxent:eq4}
    \begin{aligned}
        &\mathbb{E}_{q(\theta'|y^{(i)}(\lambda),\lambda)}[\nabla_\lambda\log l(y^{(i)}(\lambda)|\theta',\lambda)] 
        \\&\quad\quad\quad\quad\quad\quad\quad\quad
        \approx  \frac{1}{N}\sum_{j=1}^N \nabla_\lambda\log l(y^{(i)}(\lambda)|\theta^{(i,j)},\lambda).
    \end{aligned}
\end{equation}
Combining Eq.~\eqref{maxent:eq3} and Eq.~\eqref{maxent:eq4} yields an unbiased estimator of EIG gradient in theory.
In reality however, one often relies on Markov Chain Monte Carlo (MCMC) methods to sample the posterior distribution, which draws biased samples from the posterior distribution. 
We refer to the method as unbiased estimation of EIG gradient with MCMC (UEEG-MCMC). The simulation cost of a single gradient estimation for UEEG-MCMC is $O(M\times L)$, where $L$ is the number of simulations used to perform MCMC. 
We reinstate that a finite-length MCMC can not produce unbiased samples from the posterior and as such it causes bias in the gradient estimator. 
Nevertheless, we emphasize that the bias lies in the samples 
and the estimator itself is unbiased provided that samples are generated perfectly from the posterior. 
As will be shown in the numerical examples, the bias due to MCMC 
 is often much smaller than those in other methods especially for problems with large EIG values.  
Moreover, while we adopt MCMC for sampling the posterior here, the proposed method can be implemented 
with any sampling methods. To this end, if more effective sampling methods are available, they can be used instead of MCMC to reduce the estimation bias. 


\subsection{Biased Estimation of EIG Gradient with Atomic Priors}
Generating an observation from or evaluating the likelihood for each new parameter sample requires to simulate the physical model considered once more. For expensive physical models, this constitutes
the most significant computational cost.

In this section, we show how to obtain a simulation-efficient approach using \textit{atomic} priors.
Suppose a finite set of parameter-noise pairs $\Omega = \{(\theta^{(i)}, \epsilon^{(i)})\}_{i=1}^M$ are generated, where $(\theta^{(i)}, \epsilon^{(i)}) \sim \pi_\mathrm{\theta}(\theta)\pi_\mathrm{\epsilon}(\epsilon)$, and we denote $\Theta = \{\theta^{(i)}\}_{i=1}^M$. Replacing the sampling distributions $\pi_\mathrm{\theta}(\theta)\pi_\mathrm{\epsilon}(\epsilon)$ by $U_{\Omega}$ and $\pi_\mathrm{\theta}(\theta')$ by $U_{\Theta}$, where $U$ denotes the uniform distribution on the given set, we can approximate the sampling distribution over which the expected value is taken in Eq.~\eqref{maxent:eq2} as
\begin{equation}\label{maxent:eq5}
    \begin{aligned}
        &\pi_\mathrm{\theta}(\theta)\pi_\mathrm{\epsilon}(\epsilon) q(\theta'|g(\theta,\epsilon,\lambda),\lambda)
        \\
        =& \frac{\pi_\mathrm{\theta}(\theta)\pi_\mathrm{\epsilon}(\epsilon)\pi_\mathrm{\theta}(\theta')l(g(\theta,\epsilon,\lambda)|\theta',\lambda)}{\int\pi_\mathrm{\theta}(\theta')l(g(\theta,\epsilon,\lambda)|\theta',\lambda) d\theta'}\\
        \approx & \sum_{i=1}^M \frac{\sum_{j=1}^M \delta_{\theta^{(i)}}(\theta)\delta_{\epsilon^{(i)}}(\epsilon)\delta_{\theta^{(j)}}(\theta')l(y^{(i)}(\lambda)|\theta^{(j)},\lambda)}{\sum_{j=1}^M l(y^{(i)}(\lambda)|\theta^{(j)},\lambda)}.
    \end{aligned}
\end{equation}
Given this approximation of sampling distribution, we finally obtain a biased estimator of EIG gradient
\begin{equation}\label{maxent_eq7}
    \begin{aligned}
        &\nabla_\lambda U(\lambda) 
        \\
        \approx&
    \sum_{i=1}^M\frac{\sum_{j=1}^M l(y^{(i)}(\lambda)|\theta^{(j)},\lambda) \nabla_\lambda\log \big[\frac{l(y^{(i)}(\lambda)|\theta^{(i)},\lambda)}{l(y^{(i)}(\lambda)|\theta^{(j)},\lambda)}\big]}{\sum_{j=1}^M l(y^{(i)}(\lambda)|\theta^{(j)},\lambda)}.
    \end{aligned}
\end{equation}
we refer to it as biased estimation of EIG gradient with atomic priors (BEEG-AP). As it requires only one batch of parameter samples for each gradient estimation, the simulation cost amounts to $O(M)$.

\subsection{{Unifying BEEG-AP and NMC}}
To provide a more comprehensive understanding of BEEG-AP, this section reveals its close connection to NMC. Indeed, BEEG-AP can be regarded as an approach that directly computes the gradient of the NMC estimator with sample reuse technique (srNMC) in \cite{huan2013simulation}. We start this section by revisiting the concepts of NMC and srNMC. 

The naïve NMC estimates the EIG as 
\begin{equation}
    U(\lambda) \approx \frac{1}{M}\sum_{i=1}^M \log\frac{l(y^{(i)}(\lambda)|\theta^{(i)}, \lambda)}{\frac{1}{N}\sum_{j=1}^N l(y^{(i)}(\lambda)|\theta^{(i,j)}, \lambda)},
\end{equation}
where $\theta^{(i)}\sim \pi_\mathrm{\theta}(\theta)$, $\epsilon^{(i)} \sim \pi_\mathrm{\epsilon}(\epsilon)$ and $\theta^{(i,j)}\sim \pi_\mathrm{\theta}(\theta)$. This approximation requires a  simulation cost of $O(M\times N)$. To reduce the cost to $O(M)$, Huan \& Marzouk \cite{huan2013simulation} propose reusing the batch of prior samples for the outer Monte Carlo sum in all inner Monte Carlo estimations (i.e., $\theta^{(i,j)} = \theta^{(j)}$ and $N=M$). This yields a more simulation-efficient estimator of EIG
\begin{equation}\label{maxent:eq6}
    \widehat{U}_{srNMC}^{M}(\lambda) = \frac{1}{M}\sum_{i=1}^M \log\frac{l(y^{(i)}(\lambda)|\theta^{(i)}, \lambda)}{\frac{1}{M}\sum_{j=1}^M l(y^{(i)}(\lambda)|\theta^{(j)}, \lambda)}.
\end{equation}
This estimator is also related to the InfoNCE with a tractable conditional \cite{oord2018representation, poole2019variational}, often utilized for the mutual information estimation in representation learning. In addition, similar sample reuse techniques have been applied to portfolio risk measurement problems \cite{zhang2022sample, feng2022sample}. 

Now it is evident that the BEEG-AP can be directly derived from the gradient of Eq.~\eqref{maxent:eq6} w.r.t. $\lambda$. This observation allows us to explore the theoretical behavior of the BED with BEEG-AP by investigating the convergence properties of the srNMC. The original paper of \cite{huan2013simulation} only provides a simple numerical study of the bias of srNMC. Here, we give a more rigorous convergence analysis as the following theorems.
\begin{theorem}\label{maxent:thm2}
    The expectation of $\widehat{U}_{srNMC}^{M}(\lambda)$ satisfies the following: 
    \begin{enumerate}
        \item  $\mathbb{E}[\widehat{U}_{srNMC}^{M}(\lambda)]$ is a lower bound on $U(\lambda)$ for any $M>0$.
        \item  $\mathbb{E}[\widehat{U}_{srNMC}^{M}(\lambda)]$ is monotonically increasing in $M$, i.e.,
        $\mathbb{E}[\widehat{U}_{srNMC}^{M_1}(\lambda)]\leq\mathbb{E}[ \widehat{U}_{srNMC}^{M_2}(\lambda)]$ for $0\leq M_1\leq M_2$. 
    \end{enumerate}
\end{theorem}

\begin{theorem}\label{maxent:thm3}
    If $l(g(\theta,\epsilon,\lambda)|\theta', \lambda)$ is bounded away from 0 and uniformly bounded from above (i.e., $C_1\leq l(g(\theta,\epsilon,\lambda)|\theta', \lambda)\leq C_2$ a.s. for some positive constants $C_1$ and $C_2$), then
    the mean squared error of $\widehat{U}_{srNMC}^{M}(\lambda)$ converges to 0 at rate $O (1/M)$.
\end{theorem}

Theorem~\ref{maxent:thm2} indicates that the srNMC provides a lower bound estimation for the EIG and the gap can be increasingly narrowed as $M$ increases. This suggests that BED with BEEG-AP aims to use stochastic gradients to maximize a lower bound of the ground-truth EIG. Theorem~\ref{maxent:thm3} further establishes the consistency of srNMC and obtains a linear convergence rate of $O(1/M)$ under certain assumptions, indicating that the optimization objective of BED with BEEG-AP can be sufficiently close to true EIG as we increase the number of samples $M$. However, the following Theorem~\ref{maxent:thm4} suggests that achieving a neglected error is challenging in practice when the ground-truth EIG is large, even when these assumptions are admitted. 
\begin{theorem}\label{maxent:thm4}
    For any $C$ satisfying $0\leq C\leq U(\lambda)/2$, if $M\leq \exp(U(\lambda)/2)$, we have 
    \begin{equation}
        U(\lambda) - \widehat{U}_{srNMC}^{M}(\lambda)>C.
    \end{equation}
\end{theorem}
Indeed, this theorem tells us that the simulation cost required grows exponentially with the ground-truth EIG to achieve a reasonable error bound.

\section{Experiments}\label{maxentsec:exps}
A large variety of BED methods could exhibit poor performance due to the presence of large EIG values in the experiments to be designed. Before diving into numerical demonstrations (see \texttt{https://github.com/ziq-ao/GradEIG} for the research code and  Appendix for further details of our experiments), we list a few bench-marking approaches and briefly discuss the above limitation they have in common. 

 \textbf{PCE.} Prior contrastive estimation (PCE) \cite{foster2020unified} estimates the EIG as 
    \begin{equation}
    \widehat{U}_{PCE}^{M,N}(\lambda) = \frac{1}{M}\sum_{i=1}^M \log\frac{l(y^{(i)}(\lambda)|\theta^{(i)}, \lambda)}{\frac{1}{N+1}\sum_{j=0}^N l(y^{(i)}(\lambda)|\theta^{(i,j)}, \lambda)},
    \end{equation}
    where $\theta^{(i)}, \epsilon^{(i)}\sim \pi_{\theta}(\theta)\pi_{\epsilon}(\epsilon)$, $\theta^{(i,0)}=\theta^{(i)}$ and $\theta^{(i,j)}\sim \pi_{\theta}(\theta)$ for $j=1,...,N$. In contrast to srNMC, PCE only reuses one outer Monte Carlo sample in each inner Monte Carlo estimation, resulting in a simulation cost of $O(M\times N)$. It is easy to check that $\widehat{U}_{PCE}^{M,N}(\lambda)$ does not exceed $\log N$, so PCE shares a similar result to the one stated in Theorem~\ref{maxent:thm4}, implying its simulation inefficiency for estimating large EIG values.
    
    \textbf{ACE.} To improve the inner Monte Carlo in the denominator, adaptive contrastive estimation (ACE) \cite{foster2020unified} introduces a posterior inference network $q_{\phi}$ parameterized by $\phi$ and use it as the proposal distribution for sampling, i.e.,
    \begin{equation}
    \begin{aligned}
        &\widehat{U}_{ACE}^{M,N}(\lambda) 
        \\
        &\quad = \frac{1}{M}\sum_{i=1}^M \log\frac{l(y^{(i)}(\lambda)|\theta^{(i)}, \lambda)}{\frac{1}{N+1}\sum_{j=0}^N \frac{\pi_{\theta}(\theta^{i,j}) l(y^{(i)}(\lambda)|\theta^{(i,j)}, \lambda)}{q_{\phi}(\theta^{(i,j)}|y^{(i)}(\lambda))}},
    \end{aligned}
    \end{equation}
    where $\theta^{(i)}, \epsilon^{(i)}\sim \pi_{\theta}(\theta)\pi_{\epsilon}(\epsilon)$, $\theta^{(i,0)}=\theta^{(i)}$ and $\theta^{(i,j)}\sim q_{\phi}(\theta^{(i,j)}|y^{(i)}(\lambda))$ for $j=1,...,N$. Given this adaptive estimator, the network parameter $\phi$ and design variable $\lambda$ are then optimized jointly. However, learning the posterior inference network can be challenging when there are strong dependencies between the conditional (the observations) and target variables (the parameters). This occurs when the ground-truth EIG values are large (i.e., there is a high mutual information between parameters and observations). In practice, we observe that general-purpose conditional density networks (such as Mixture Density Network \cite{bishop1994mixture} and Normalizing Flows \cite{papamakarios2021normalizing}) usually fail or run indefinitely for invalid values during training in this case.
    
    \textbf{GradBED.} Gradient-based Bayesian Experimental Design (GradBED) \cite{kleinegesse2021gradient} designs experiments by optimizing a variational lower bound of mutual information between parameters and observations. A variety of candidates of lower bounds can be found in \cite{kleinegesse2021gradient}. In this paper, we only consider the following NWJ estimator \cite{nguyen2010estimating}:
    \begin{equation}
        \begin{aligned}
             &\widehat{U}_{NWJ}^{M}(\lambda) 
             \\
             =& \frac{1}{M}\sum_{i=1}^M \big[T_{\psi}(\theta^{(i)}, y^{(i)}(\lambda))\! -\! \frac{1}{e}\exp{(T_{\psi}(\theta^{(i)},  y'^{(i)}(\lambda))}\big],
        \end{aligned}
    \end{equation}
    where $\theta^{(i)}, \epsilon^{(i)}\sim \pi_{\theta}(\theta)\pi_{\epsilon}(\epsilon)$, $\theta'^{(i)}, \epsilon'^{(i)}\sim \pi_{\theta}(\theta)\pi_{\epsilon}(\epsilon)$, $y'^{(i)}(\lambda)) = g(\theta'^{(i)},\epsilon'^{(i)},\lambda)))$ and $T_{\psi}$ is a neural network parameterized by $\psi$. The simulation cost is $O(M)$ for GradBED. During optimization, the network parameter $\psi$ and design variable $\lambda$ are updated simultaneously. As studied in \cite{song2019understanding}, the variance of certain variational mutual information estimators, including NWJ, could grow exponentially with the ground-truth mutual information (or EIG in Bayesian experimental design literature) and thereby lead to poor designs.

\subsection{EIG Gradient Estimation Accuracy}
We start by examining the empirical convergence properties of the proposed estimators. We consider a Bayesian linear regression model with tractable EIG gradients. 
Assume $n\times 1$ observations are generated by the following linear acquisition system
\begin{equation}
    \mathrm{y} = \mathrm{D}\theta + \epsilon
\end{equation}
where 
$\mathrm{D} = [1,\lambda', (\lambda')^2]$ 
is the design matrix obtained by the design vector $\lambda = (\lambda_1,...,\lambda_n)$, $\theta=(\theta_1, \theta_2, \theta_3)'$ are the parameters of interest and $\epsilon$ are $n\times 1$ i.i.d. noises. In Bayesian framework, we assign a Gaussian prior $\theta\sim \mathcal{N}(0,I_3)$ on the unknown parameters and a Gaussian observation noise with variance $\sigma^2$, that is, $p(\mathrm{y}|\theta) = \mathcal{N}(\mathrm{D}\theta, \sigma^2I_n)$.
The above modeling admits a closed-form representation of EIG using the entropy expressions of multivariate normal distributions\cite{ahmed1989entropy}, that is,
\begin{equation}
    U(\lambda)=\frac{1}{2}(\log\frac{|\mathrm{D}\mathrm{D}'+\sigma^2I_n|}{|\sigma^2I_n|}).
\end{equation}
The EIG gradient can then be analytically derived or directly computed by automatic differentiation frameworks. 

In this study, we estimate and compare the biases of three approaches (BEEG-AP, UEEG-MCMC and PCE) with a large number of repeated trials. The scatter plots in Fig.~\ref{fig:gradient_accuracy} shows the comparisons of these estimated biases across 20 independent designs.
From the top of Fig.~\ref{fig:gradient_accuracy}, it is evident that the UEEG-MCMC has lower bias with the ground-truth EIG increases, outperforming the other two methods. On the other hand, the bottom of Fig.~\ref{fig:gradient_accuracy} exhibits that, while BEEG-AP requires fewer simulation costs, it yields a comparable level of bias to PCE. 

\begin{figure}
    \centering
    \includegraphics[width=0.15\textwidth]{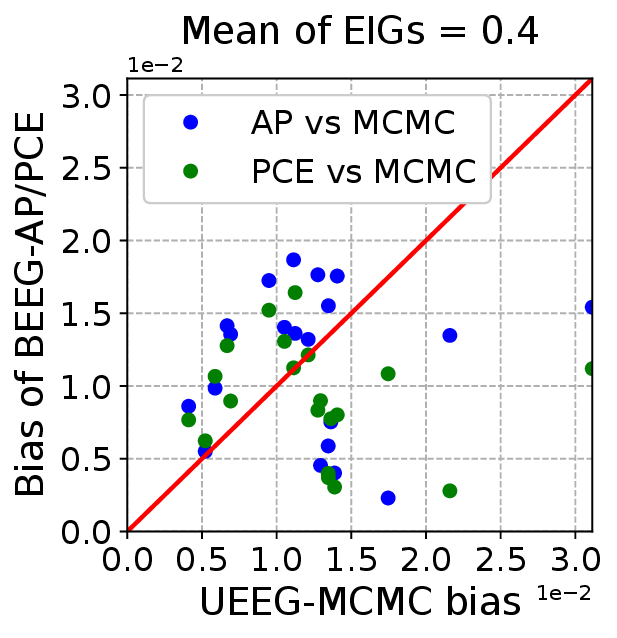}
    \includegraphics[width=0.15\textwidth]{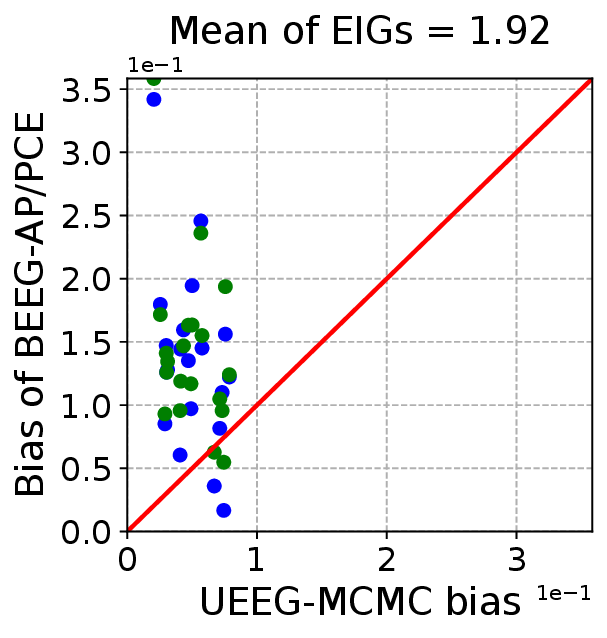}
    \includegraphics[width=0.15\textwidth]{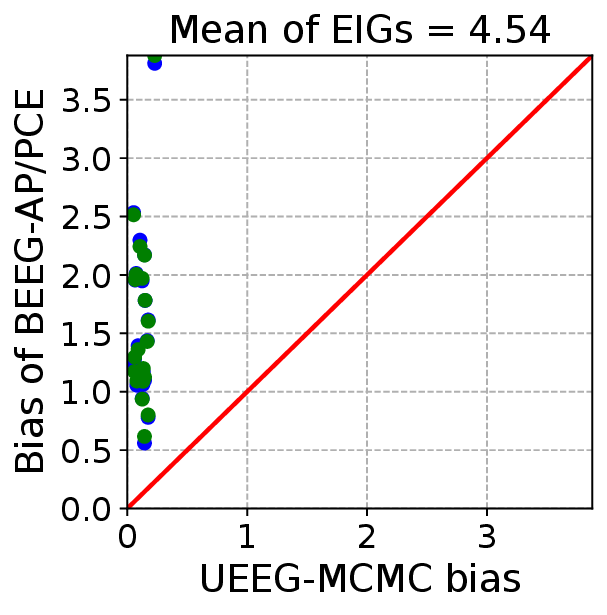}
    \includegraphics[width=0.15\textwidth]{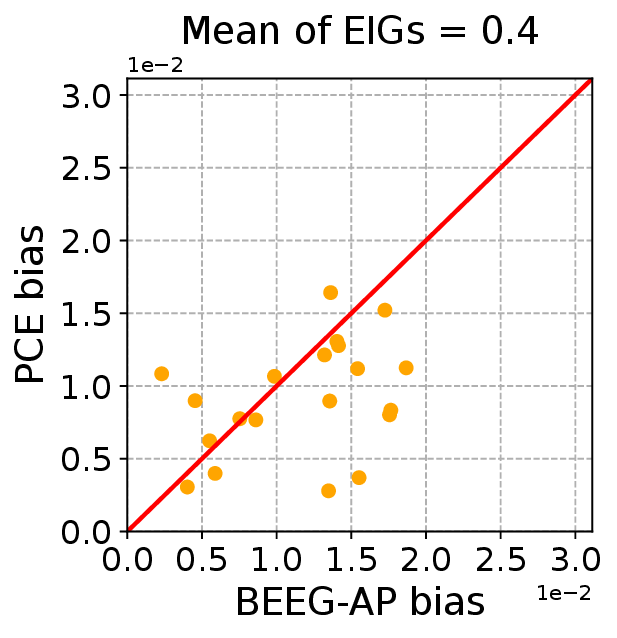}
    \includegraphics[width=0.15\textwidth]{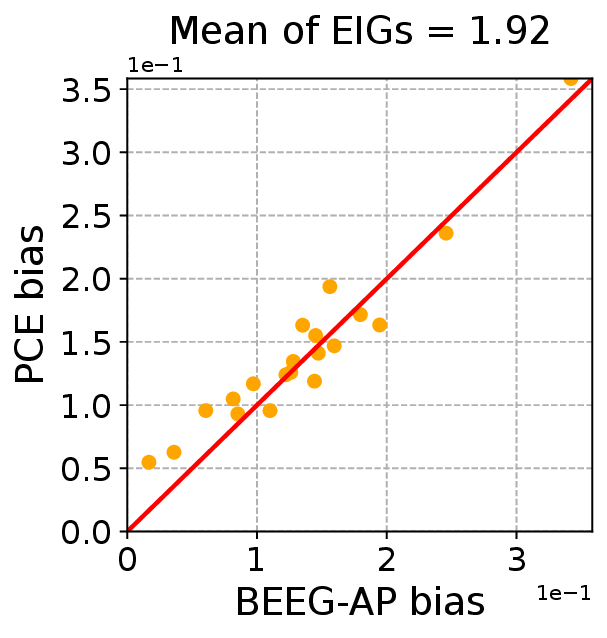}
    \includegraphics[width=0.15\textwidth]{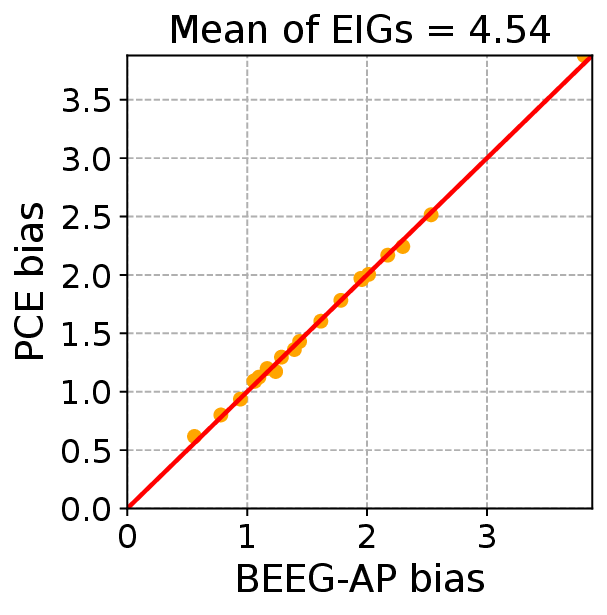}
    \caption{Top: the estimated biases of BEEG-AP and PCE versus those of UEEG-MCMC for 20 independent designs. Bottom: the estimated biases of PCE versus those of BEEG-AP for 20 independent designs.}
    \label{fig:gradient_accuracy}
\end{figure}

\subsection{A Toy Algebraic Model}
In this experiment, we consider a toy problem with a single optimal design. The model is given by the following nonlinear map:
\begin{equation}
    \bigg[\begin{matrix}
        y_1 \\y_2
    \end{matrix}\bigg]
    \!\!=\!\!
    \bigg[
    \begin{matrix}
        0.5\theta^3 d_1 + \theta \exp(-|0.2-0.5d_1|)+d_1^2\\
        0.5\theta^3 (d_2\!+\!1.6)\! +\! \theta \exp(-|0.6\!+\!\!0.5d_2|)\!+\!d_2^2
    \end{matrix}
    \bigg]
    \!+\!
    \bigg[
    \begin{matrix}
        \epsilon_1\\
        \epsilon_2
    \end{matrix}
    \bigg],
\end{equation}
where $\theta$ is the model parameter, $d_i$ are design variables and $\epsilon_i$ are independent observation noises. We assign a uniform prior $\theta \sim \mathrm{Unif}(0,1)$ on the parameter and restrict the design variables in the interval of $[0, 1]$. We conduct two experiments with different noise terms, i.e., a large noise scenario with $\epsilon_i \sim \mathcal{N}(0, (0.1)^2)$ and a small noise scenario with $\epsilon_i \sim \mathcal{N}(0, (0.0001)^2)$. By setting different noise terms, we can create scenarios to represent the small and large EIG cases and observe how the different methods perform under these conditions.

We applied five methods to this problem: BEEG-AP, UEEG-MCMC, ACE, PCE and GradBED. To mitigate the impact of randomness, we perform 20 independent runs for each method. The results for both the large and small noise settings are depicted in Fig.~\ref{fig:toy}. From the figures we can see the final designs obtained by BEEG-AP and UEEG-MCMC are more concentrated compared to the designs generated by the other three methods. In particular, in the small noise case, UEEG-MCMC stands out as the only method where all designs eventually concentrate on a single point. Then we use different metrics to judge the quality of the designs obtained for the two settings. In the large noise case, we apply NMC with large samples to obtain high-quality estimations of the EIGs. Fig.~\ref{fig:toy_val_large} shows the estimated EIGs throughout the entire design space. Remarkably, we observe that the only optimal design identified by the estimations aligns with the the results obtained by BEEG-AP and UEEG-MCMC. In the small noise case, utilizing NMC for reliable EIG estimations becomes impractical due to the large simulation budget required. We therefore resort to using the posterior entropy as the metric to evaluate the quality of the designs, and the results are plotted in Fig.~\ref{fig:toy_val_small}. From the figure, it is evident that BEEG-AP and UEEG-MCMC produce designs with
smaller posterior entropy. Specifically, UEEG-MCMC yields even smaller posterior entropy than BEEG-AP, which demonstrates the superior performance  of UEEG-MCMC in large EIG case.

\begin{figure}
    \centering
    \includegraphics[width=0.23\textwidth]{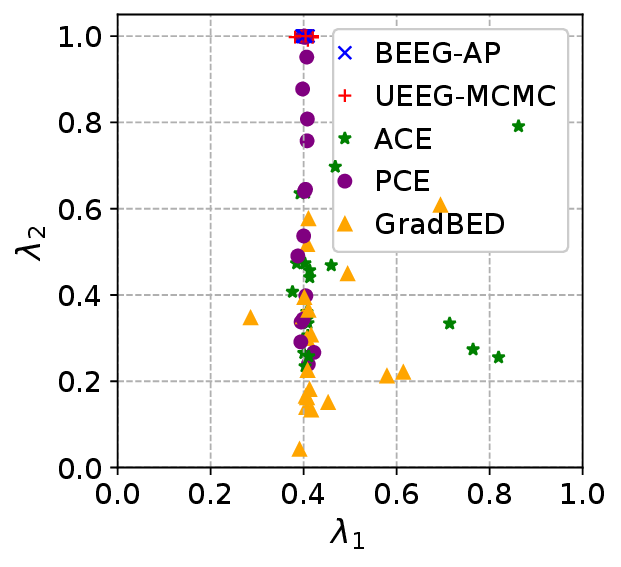}
    \includegraphics[width=0.23\textwidth]{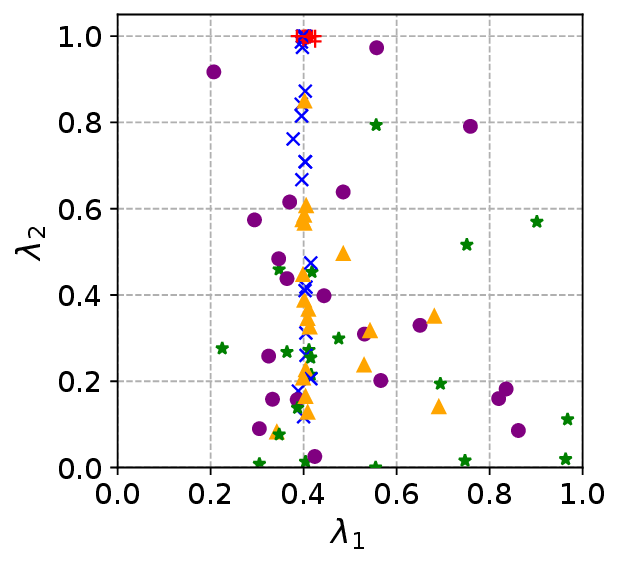}
    \caption{The final designs of 20 independent trials for large noise setting (left) and small noise setting (right).}
    \label{fig:toy}
\end{figure}

\begin{figure}
    \centering
    \begin{minipage}{0.2\textwidth}
        \centering
        \includegraphics[width=\textwidth]{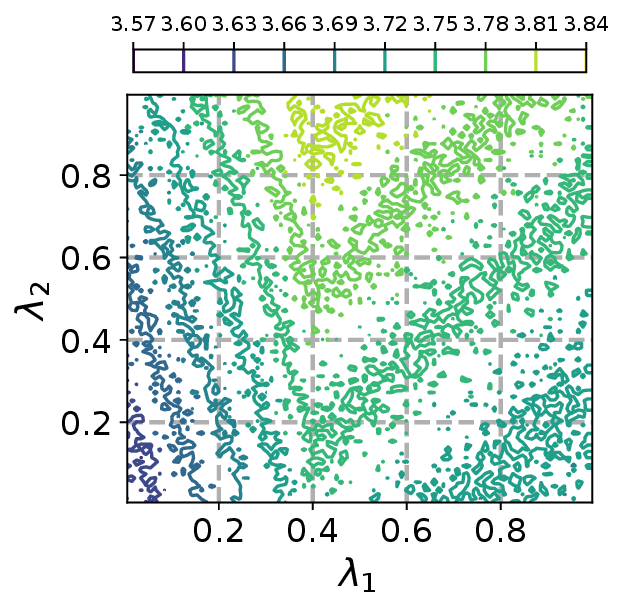}
        \caption{Estimates of EIG for large noise setting. }
        \label{fig:toy_val_large}
    \end{minipage}
    \hspace{0.02\textwidth}
    \begin{minipage}{0.24\textwidth}
        \centering
        \includegraphics[width=\textwidth]{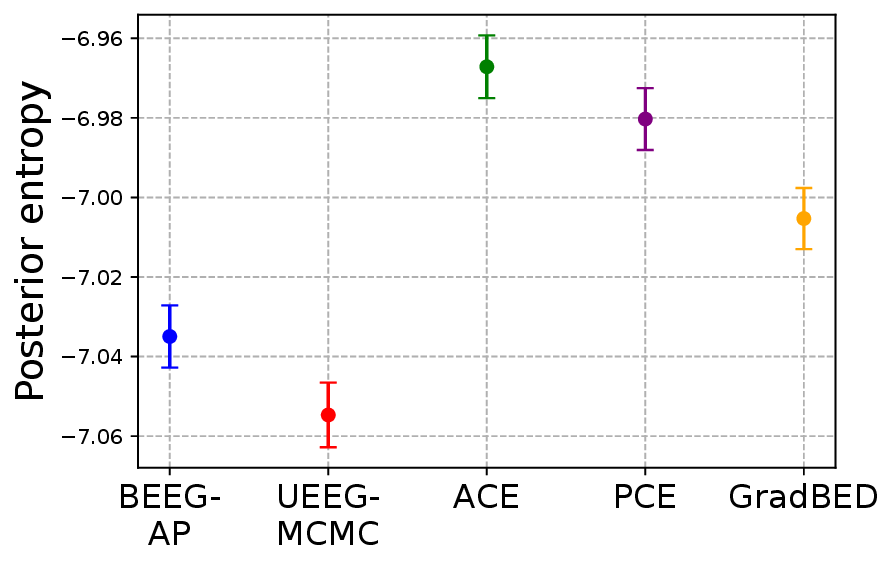}
        \caption{The posterior entropy for small noise setting. Shown are the means of entropy with their standard error bars.}
        \label{fig:toy_val_small}
    \end{minipage}
\end{figure}



\subsection{Pharmacokinetic (PK) Model}
Next, we consider an experimental design problem for PK studies. These studies aim to comprehend the underlying kinetics of a drug, shedding light on how it is absorbed, distributed, metabolized, and eliminated within the body over time. To achieve this understanding, it is a common practice to collect blood samples from the study subjects. However, the process of blood sampling involves various practical constraints, including financial limitations, participant burden, and ethical considerations. Therefore, researchers must strategically design the sampling strategy to gather meaningful information while minimizing the number of blood samples required.
Here, we focus on the PK model introduced by \cite{ryan2014towards}. The model under consideration consists of three parameters of interest ${\theta}=(k_a, k_e, V)$: the absorption rate constant $k_a$, the elimination rate constant $k_e$, and the volume of distribution V which represents the theoretical volume that the drug would need to occupy to achieve the current concentration in the blood plasma. The drug concentration of blood sample taken at time $t$, denoted as $y_t$, follows
\begin{equation}
    y_t = \frac{D}{V} \cdot \frac{k_a}{k_a - k_e} \cdot (e^{-k_e t} - e^{-k_a t}) \cdot (1 + \epsilon_{1t}) + \epsilon_{2t},
\end{equation}
where $D = 400$ is the fixed dose administered at the beginning of the experiment, $\epsilon_{1t}$ and $\epsilon_{2t}$ are the multiplicative and additive Gaussian
noises respectively. As in \cite{ryan2014towards}, we assign a log-normal prior on $\theta$
\begin{equation}
\log \theta \sim N\left[\left(\begin{array}{c}
\log (1) \\
\log (0.1) \\
\log (20)
\end{array}\right),\left(\begin{array}{ccc}
0.05 & 0 & 0 \\
0 & 0.05 & 0 \\
0 & 0 & 0.05
\end{array}\right)\right].
\end{equation}
The design variables are assumed to be the 10 blood sampling times $(d_1,..., d_{10})$, $d_i\geq 0$ for $i=1,...,10$, and the corresponding drug concentrations at these times $(y_{d_1}, ..., y_{d_{10}})$ form the observations. 

Again, we create a small EIG setting with $\epsilon_{1t}\sim \mathcal{N}(0, 0.01)$ and $\epsilon_{2t}\sim \mathcal{N}(0, 0.1)$, and a large EIG setting with $\epsilon_{1t}=0$ and $\epsilon_{2t}\sim \mathcal{N}(0, 0.001)$. In small EIG scenario we use large samples to compute high-quality NMC estimations to validate various methods, while in large EIG scenario we compare the posterior entropies obtained by them.
Fig.~\ref{fig:pk_gt} shows BEEG-AP outperforms other methods in terms of convergence rate in small EIG scenario. This can also be supported by examining the convergence histories of all methods, as shown in 
Fig.~B.1
in the Appendix. However, in large EIG scenario, UEEG-MCMC achieve the best performance as the validation experiments indicate in 
Fig.~\ref{fig:pk_val} 
and the convergence histories show in
Fig.~B.2
in the Appendix. It should be mentioned that in this case ACE fails to learn a posterior inference network due to the strong dependencies between the observations and the parameters.

\begin{figure}
    \centering
    \begin{minipage}{0.22\textwidth}
        \centering
        \includegraphics[width=\textwidth]{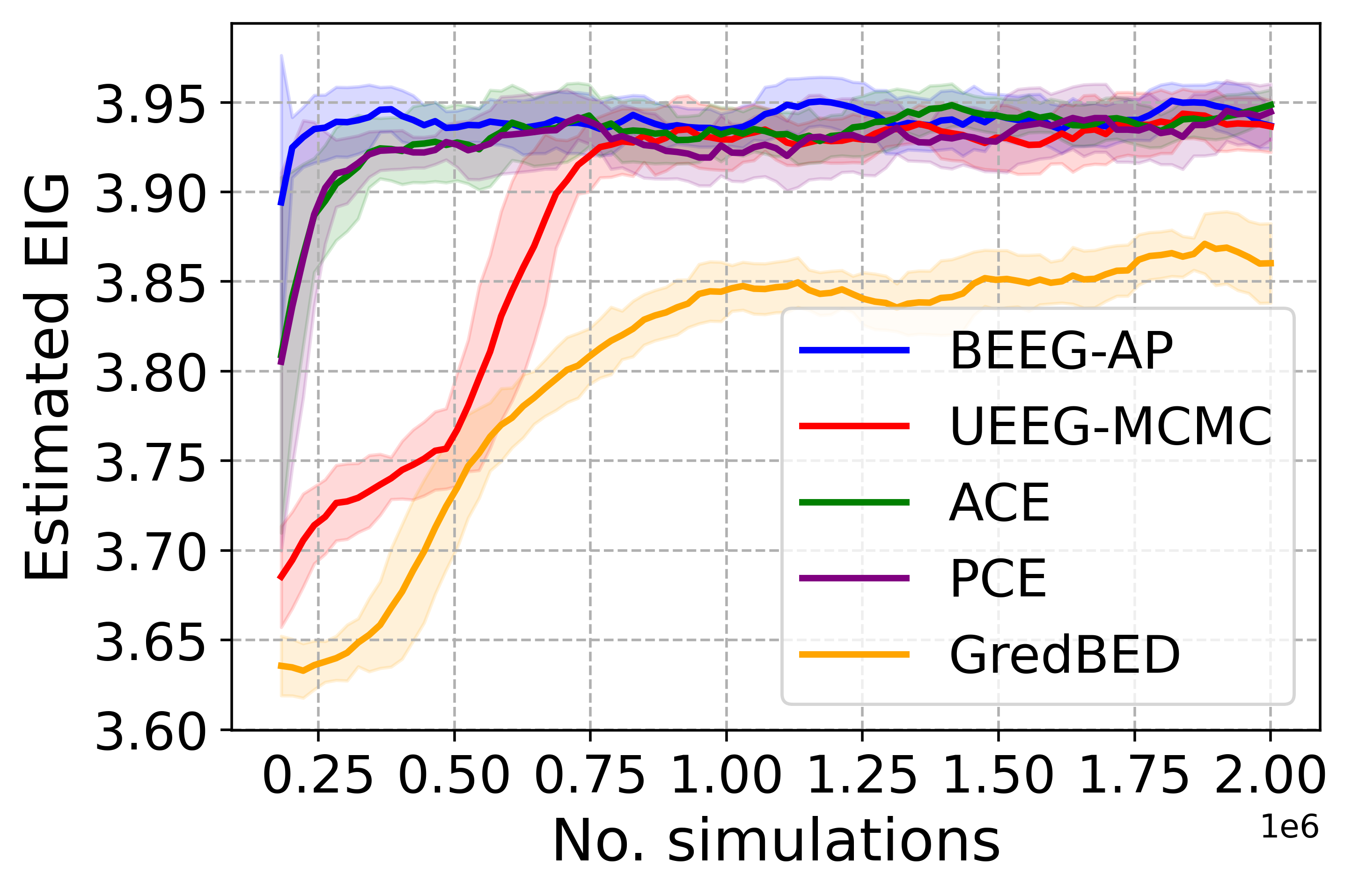}
        \vspace{-0.5cm}
        \caption{Optimization of EIG for PK model with multiplicative noise $\mathcal{N}(0, 0.01)$ and additive noise $\mathcal{N}(0, 0.1)$ as a function of number of simulations. Shown are the moving averages with the standard error bars.}
        \label{fig:pk_gt}
    \end{minipage}
    \hspace{0.02\textwidth}
    \begin{minipage}{0.22\textwidth}
        \centering
        \includegraphics[width=\textwidth]{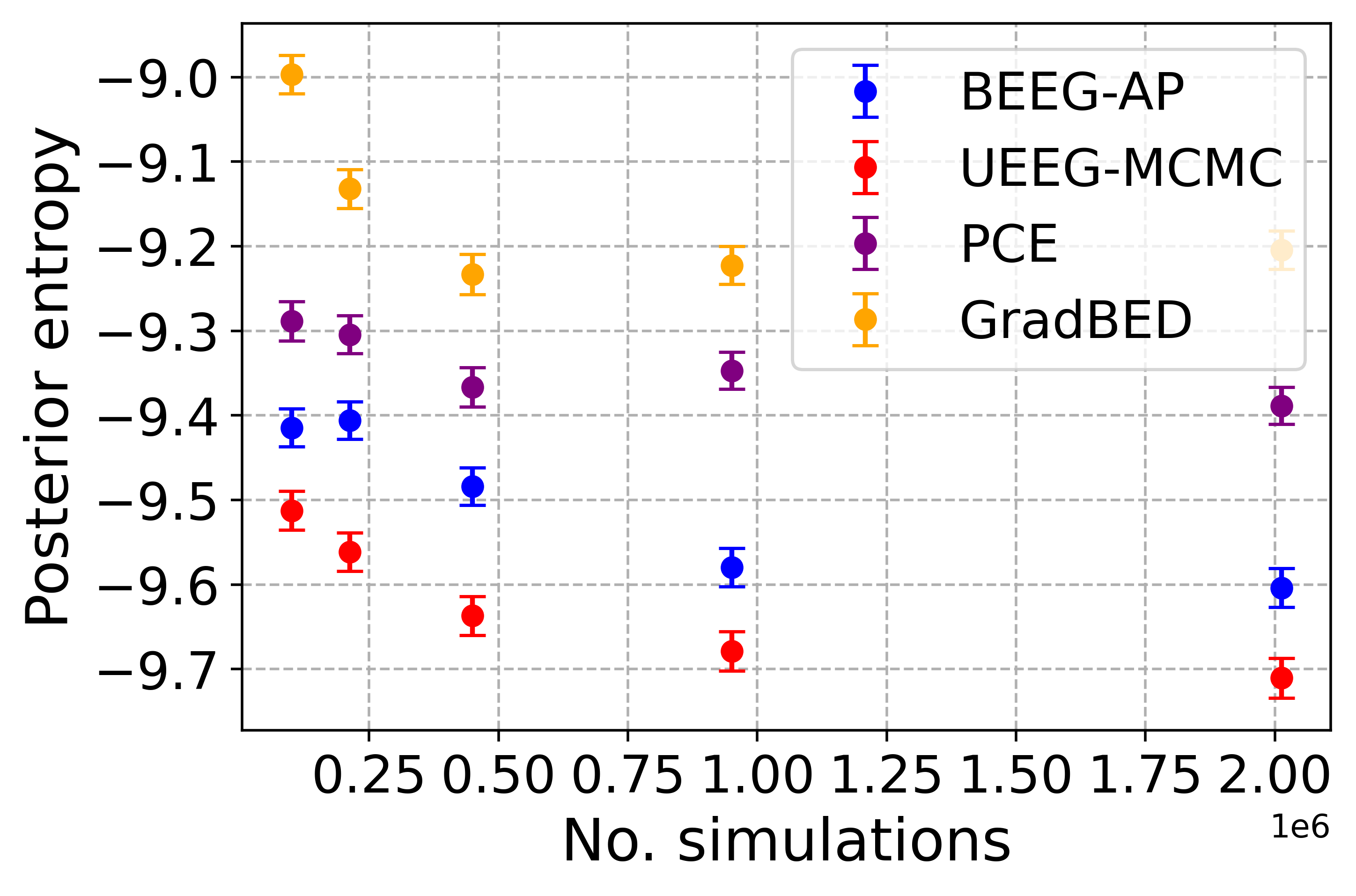}
        \vspace{-0.5cm}
        \caption{The posterior entropy for PK model with additive noise $\mathcal{N}(0, 0.001)$ as a function of number of simulations. Shown are the means of entropy with their standard error bars.\\~}
    \label{fig:pk_val}
    \end{minipage}
\end{figure}



\subsection{Signal Transducer and Activator of Transcription 5 (STAT5) model}
Finally we aim to design the measurement times for a dynamic system modeled by ordinary differential equations (ODEs). 
We take the mathematical model of the core module of the Janus family of kinases (JAK)–signal
transducer and activator of transcription (STAT) pathway in \cite{swameye2003identification} as a case study. The core module of the JAK-STAT pathway is represented by the latent transcription factor STAT5, and the dynamics of STAT5 populations $x_1$, $x_2$, $x_3$ and $x_4$ can be described by four coupled ODEs (see Appendix for full details of the ODEs). The rate constants $k_1$, $k_2$ and the delay parameter $\tau$ are the three model parameters to be inferred from measured data.

It is experimentally challenging to directly measure distinct STAT5 populations separately. Instead, one can measure the amount of
tyrosine phosphorylated STAT5 $y_1=s_1(x_2+x_3)$ and the total amount of STAT5  $y_2=s_2(x_1+x_2+x_3)$. We assume the scaling parameters to be $s_1 = 0.33$ and $s_2 = 0.26$.
We assign a uniform prior on $\theta = (k_1, k_2, \tau)$ with lower range $[0.5,0.05,4.0]$ and upper range $[3.0, 0.2, 10.0]$. The objective of the experimental design is to allocate 16 measurement times for STAT5 populations over a time span from 0 to 60 minutes, which yields 32 experimental measurements in total.

In this application, we set two levels of additive Gaussian observation noises, $\mathcal{N}(0, 10^{-4})$ and $\mathcal{N}(0, 10^{-6})$, representing the small and large EIG scenarios respectively. In both cases, we assess the quality of designs by the posterior entropies obtained. The results for the small EIG scenario depicted in Fig.~\ref{fig:stats5_val_0} indicate that, both BEEG-AP and UEEG-MCMC outperform other methods and BEEG-AP appears the best. However, UEEG-MCMC demonstrates the best performance in the large EIG scenario, while BEEG-AP's performance falls below that of GradBED, as shown in Fig.~\ref{fig:stats5_val_1}. ACE fails in both scenario. The convergence histories of all approaches can be found in 
Fig.~B.3
and 
Fig.~B.4
in the Appendix.

\begin{figure}
    \centering
    \begin{minipage}{0.22\textwidth}
        \centering
        \includegraphics[width=\textwidth]{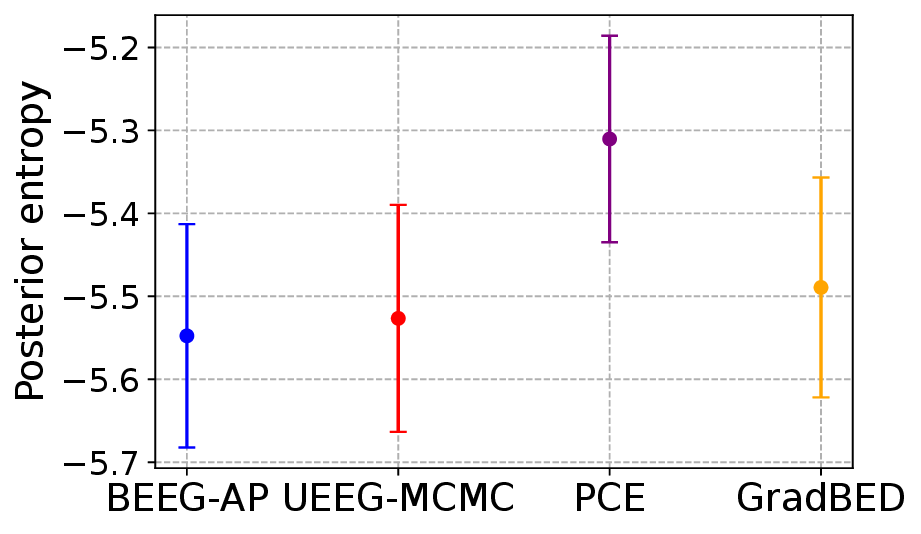}
        \caption{The posterior entropy for STAT5 model with additive noise $\mathcal{N}(0, 10^{-4})$ for designs obtained. Shown are the means of entropy with their standard error bars.}
        \label{fig:stats5_val_0}
    \end{minipage}
    \hspace{0.02\textwidth}
    \begin{minipage}{0.22\textwidth}
        \centering
        \includegraphics[width=\textwidth]{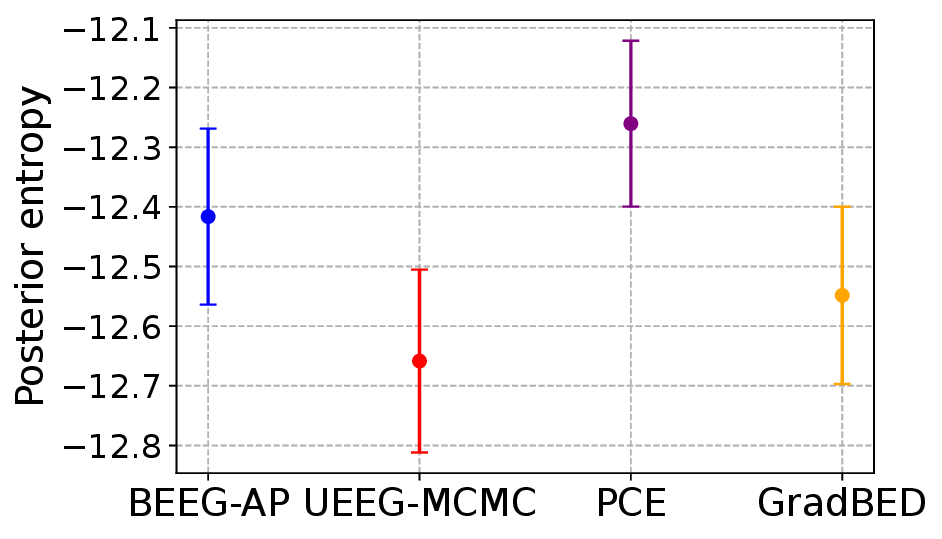}
        \caption{The posterior entropy for STAT5 model with additive noise $\mathcal{N}(0, 10^{-6})$ for designs obtained. Shown are the means of entropy with their standard error bars.}
    \label{fig:stats5_val_1}
    \end{minipage}
\end{figure}

\section{Conclusion}\label{maxentsec:con}
In this work we have proposed two approaches, UEEG-MCMC and BEEG-AP, to Bayesian experimental design based on EIG gradient estimation. We use MCMC sampling techniques to build the gradient estimation for UEEG-MCMC, while BEEG-AP approximates the EIG gradients with atomic priors. Both of them are straightforward to implement and have demonstrated improved performance  compared to bench-marking methods.

Our theoretical analysis aligns well with the numerical results. Specifically, BEEG-AP exhibits superior simulation efficiency when dealing with problems that have small ground-truth EIG values, making it a favorable option in such cases. On the other hand, UEEG-MCMC shows robustness across various EIG levels, making it suitable for a broader range of experimental scenarios. We believe the work provides researchers and practitioners with promising tools and guidance to optimize their experiments and make informed decisions across different domains.

\section{Acknowledgments}
This work was partially supported by the China Scholarship Council (CSC).

\bibliography{bibmaxent}

\ifshowappendix
\onecolumn
\appendix

\setcounter{secnumdepth}{2} 
\numberwithin{equation}{section}
\numberwithin{figure}{section}

\setcounter{theorem}{0}
\setcounter{corollary}{0}
\setcounter{lemma}{0}
\section{Proofs of Results}
\begin{lemma}\label{app:lemma1}
    The gradient of the logarithm of the marginal density w.r.t. the experimental condition $\lambda$ admits the following representation:
    \begin{equation}
        \nabla_\lambda \log p(g(\theta,\epsilon,\lambda))|\lambda) = -\mathbb{E}_{q(\theta'|g(\theta,\epsilon,\lambda),\lambda)}[\nabla_\lambda \log l(g(\theta,\epsilon,\lambda)|\theta', \lambda)],
    \end{equation}
    where $q(\theta'|y,\lambda)\propto \pi_{\mathrm{\theta}}(\theta')l(y|\theta',\lambda)$ is the posterior density of parameters given the observation sample $y$.
\end{lemma}
\begin{proof}
    Let $y=g(\theta,\epsilon,\lambda)$,
    we find
    \begin{equation}
        \begin{aligned}
        & \nabla_\lambda \log p(g(\theta,\epsilon,\lambda)|\lambda) \\
        =& \nabla_\lambda \log p(y|\lambda) \\
        =& \frac{1}{p(y|\lambda)}\nabla_\lambda p(y|\lambda) \\
        =& \frac{1}{p(y|\lambda)}\nabla_\lambda\int \pi_{\mathrm{\theta}}(\theta') l(y|\theta',\lambda) d\theta' \\
        =& \frac{1}{p(y|\lambda)}\int \pi_{\mathrm{\theta}}(\theta') \nabla_\lambda l(y|\theta',\lambda) d\theta' \\
        =& \int {q(\theta'|y,\lambda)}\frac{\nabla_\lambda l(y|\theta',\lambda)}{l(y|\theta',\lambda)} d\theta' \\
        =& \int {q(\theta'|y,\lambda)}\nabla_\lambda\log l(y|\theta',\lambda) d\theta' \\
        =& \mathbb{E}_{q(\theta'|y,\lambda)}[\nabla_\lambda\log l(y|\theta',\lambda)],
        \end{aligned}
    \end{equation}
    Finally, by plugging $y=g(\theta,\epsilon,\lambda)$ back into the equation, we obtain
    \begin{equation}
        \nabla_\lambda \log p(g(\theta,\epsilon,\lambda)|\lambda) = \mathbb{E}_{q(\theta'|g(\theta,\epsilon,\lambda),\lambda)}[\nabla_\lambda\log l(g(\theta,\epsilon,\lambda)|\theta',\lambda)].
    \end{equation}
\end{proof}

\begin{theorem}\label{app:thm1}
    The gradient of the entropy $H(p(y| \lambda))$ w.r.t. the experimental condition $\lambda$ satisfies
    \begin{equation}
        \nabla_\lambda H(p(y|\lambda)) = -\mathbb{E}_{\pi_\mathrm{\theta}(\theta)\pi_\mathrm{\epsilon}(\epsilon) q(\theta'|g(\theta,\epsilon,\lambda),\lambda)}[\nabla_\lambda\log l(g(\theta,\epsilon,\lambda)|\theta',\lambda)],
    \end{equation}
    where $q(\theta'|y,\lambda)\propto \pi_{\mathrm{\theta}}(\theta')l(y|\theta',\lambda)$ is the posterior density of parameters given the observation sample $y$.
\end{theorem}
\begin{proof}
By the reparameterization trick, we have
    \begin{equation}
        \begin{aligned}
            \nabla_\lambda H(p(y|\lambda)) &= -\nabla_\lambda \mathbb{E}_{\pi_\mathrm{\theta}(\theta)\pi_\mathrm{\epsilon}(\epsilon)}[\log p(g(\theta,\epsilon,\lambda)|\lambda)]
            \\&= -\mathbb{E}_{\pi_\mathrm{\theta}(\theta)\pi_\mathrm{\epsilon}(\epsilon)}[\nabla_\lambda\log p(g(\theta,\epsilon,\lambda)|\lambda)].
        \end{aligned}
    \end{equation}
    Using Lemma~\ref{app:lemma1}, we can substitute the gradient of the logarithm of the marginal density in the above equation by the posterior expectation of the logarithm of the likelihood function.
\end{proof}

\begin{corollary}\label{app:cor1}
    The gradient of the EIG $U(\lambda)$ w.r.t. the experimental condition $\lambda$ satisfies
    \begin{equation}\label{maxent:eq7}
        \nabla_\lambda U(\lambda) =
        \mathbb{E}_{\pi_\mathrm{\theta}(\theta)\pi_\mathrm{\epsilon}(\epsilon) q(\theta'|g(\theta,\epsilon,\lambda),\lambda)}[\nabla_\lambda\log l(g(\theta,\epsilon,\lambda)|\theta,\lambda)-\nabla_\lambda\log l(g(\theta,\epsilon,\lambda)|\theta',\lambda)],
    \end{equation}
    where $q(\theta'|y,\lambda)\propto \pi_{\mathrm{\theta}}(\theta')l(y|\theta',\lambda)$ is the posterior density of parameters given the observation sample $y$.
\end{corollary}
\begin{proof}
   Applying the reparameterization trick to the gradient of the negative conditional entropy term $\mathbb{E}_{\pi_{\mathrm{\theta}(\theta)l(y|\theta, \lambda)}}[\log l(y|\theta, \lambda)]$ w.r.t. the experimental condition $\lambda$, we have
    \begin{equation}
        \begin{aligned}            \nabla_\lambda\mathbb{E}_{\pi_{\mathrm{\theta}(\theta)l(y|\theta, \lambda)}}[\log l(y|\theta, \lambda)]
        &=
        \mathbb{E}_{\pi_\mathrm{\theta}(\theta)\pi_\mathrm{\epsilon}(\epsilon)}[\nabla_\lambda\log l(g(\theta,\epsilon,\lambda)|\theta, \lambda)]\\
        & = \mathbb{E}_{\pi_\mathrm{\theta}(\theta)\pi_\mathrm{\epsilon}(\epsilon)q(\theta'|g(\theta,\epsilon,\lambda),\lambda)}[\nabla_\lambda\log l(g(\theta,\epsilon,\lambda)|\theta, \lambda)].
        \end{aligned}
    \end{equation}
    Combining the above equation with Theorem~\ref{app:thm1}, we finally get Eq.~\eqref{maxent:eq7}.
\end{proof}

\begin{theorem}\label{app:thm2}
    The expectation of $\widehat{U}_{srNMC}^{M}(\lambda)$ satisfies the following: 
    \begin{enumerate}
        \item  $\mathbb{E}[\widehat{U}_{srNMC}^{M}(\lambda)]$ is a lower bound on $U(\lambda)$ for any $M>0$.
        \item  $\mathbb{E}[\widehat{U}_{srNMC}^{M}(\lambda)]$ is monotonically increasing in $M$, i.e.,
        $\mathbb{E}[\widehat{U}_{srNMC}^{M_1}(\lambda)]\leq\mathbb{E}[ \widehat{U}_{srNMC}^{M_2}(\lambda)]$ for $0\leq M_1\leq M_2$. 
    \end{enumerate}
\end{theorem}
\begin{proof}
    To prove the first result in Theorem~\ref{app:thm2}, we first note that the outer Monte Carlo terms are identically distributed. Thus,
    \begin{equation}
        \begin{aligned}
            \mathbb{E}[\widehat{U}_{srNMC}^{M}(\lambda)] = \mathbb{E}\Bigg[ \log\frac{l(y^{(1)}|\theta^{(1)}, \lambda)}{\frac{1}{M}\sum_{j=1}^M l(y^{(1)}|\theta^{(j)}, \lambda)}\Bigg],
        \end{aligned}
    \end{equation}
    where $y^{(1)} = g(\theta^{(1)},\epsilon^{(1)},\lambda)$ and the expectation is taken over $p(y^{(1)}|\lambda)q(\theta^{(1)}|y^{(1)}, \lambda)\prod_{j=2}^M \pi_\mathrm{\theta}(\theta^{(j)})$. We proceed the rest as in \cite{foster2020unified}.
    We let $\delta = U(\lambda) - \mathbb{E}[\widehat{U}_{srNMC}^{M}(\lambda)]$, then
    \begin{equation}
        \begin{aligned}
            \delta &= \mathbb{E}\Bigg[\log\frac{\frac{1}{M}\sum_{j=1}^M l(y^{(1)}|\theta^{(j)}, \lambda)}{p(y^{(1)}|\lambda)}\Bigg] \\
            & = \mathbb{E}\Bigg[\log\frac{\frac{1}{M}\sum_{j=1}^M q(\theta^{(j)}|y^{(1)}, \lambda)\prod_{k\neq j} \pi_\mathrm{\theta}(\theta^{(k)})}{\prod_{j=1}^M \pi_\mathrm{\theta}(\theta^{(j)})}\Bigg]\\
            & = \mathbb{E}\Bigg[\log\frac{P(\theta^{(1:M)}|y^{(1)})}{\prod_{j=1}^M \pi_\mathrm{\theta}(\theta^{(j)})}\Bigg],
        \end{aligned}
    \end{equation}
    where $P(\theta^{(1:M)}|y^{(1)}) = \frac{1}{M}\sum_{j=1}^M q(\theta^{(j)}|y^{(1)}, \lambda)\prod_{k\neq j} \pi_\mathrm{\theta}(\theta^{(k)})$. Due to the permutation symmetry of the integrand over the labels $1,...,M$, the expectation keeps the same if it is instead taken over $p(y^{(1)}|\lambda)P(\theta^{(1:M)}|y^{(1)})$. Thus we have
    \begin{equation}
        \delta = \mathbb{E}\bigg[\mathrm{D_{KL}}\bigg(P(\theta^{(1:M)}|y^{(1)}) \| \prod_{j=1}^M \pi_\mathrm{\theta}(\theta^{(j)})\bigg)\bigg] \geq 0,
    \end{equation}
    where the expectation is taken over $p(y^{(1)}|\lambda)$.

    To prove the second result, as in the former proof we let $\phi = \mathbb{E}[ \widehat{U}_{srNMC}^{M_2}(\lambda)]-\mathbb{E}[\widehat{U}_{srNMC}^{M_1}(\lambda)]$. Then
    \begin{equation}
        \begin{aligned}
            \phi &= \mathbb{E}\Bigg[\log\frac{\frac{1}{M_1}\sum_{j=1}^{M_1} l(y^{(1)}|\theta^{(j)}, \lambda)}{\frac{1}{M_2}\sum_{j=1}^{M_2} l(y^{(1)}|\theta^{(j)}, \lambda)}\Bigg] \\
            & = \mathbb{E}\Bigg[\log\frac{Q(\theta^{(1:M_2)}|y^{(1)})}{P(\theta^{(1:M_2)}|y^{(1)})}\Bigg],
        \end{aligned}
    \end{equation}
    where the expectation is taken over $p(y^{(1)}|\lambda)q(\theta^{(1)}|y^{(1)}, \lambda)\prod_{j=2}^{M_2} \pi_\mathrm{\theta}(\theta^{(j)})$ and $Q(\theta^{(1:M_2)}|y^{(1)}) = \frac{1}{M_1}\sum_{j=1}^{M_1} q(\theta^{(j)}|y^{(1)}, \lambda)\prod_{k\neq j}^{M_2} \pi_\mathrm{\theta}(\theta^{(k)})$. Again, using the permutation symmetry, we can get the same expectation if the sampling distribution is taken as $p(y^{(1)}|\lambda)Q(\theta^{(1:M_2)}|y^{(1)})$ instead. Thus,
    \begin{equation}
        \phi = \mathbb{E}\bigg[\mathrm{D_{KL}}\bigg(Q(\theta^{(1:M_2)}|y^{(1)}) \| P(\theta^{(1:M_2)}|y^{(1)})\bigg)\bigg] \geq 0,
    \end{equation}
    where the expectation is taken over $p(y^{(1)}|\lambda)$.
\end{proof}

\begin{theorem}\label{app:thm3}
    If $l(g(\theta,\epsilon,\lambda)|\theta', \lambda)$ is bounded away from 0 and uniformly bounded from above (i.e., $C_1\leq l(g(\theta,\epsilon,\lambda)|\theta', \lambda)\leq C_2$ a.s. for some positive constants $C_1$ and $C_2$), then
    the mean squared error of $\widehat{U}_{srNMC}^{M}(\lambda)$ converges to 0 at rate $O (1/M)$.
\end{theorem}
\begin{proof}
    For simplicity, we denote $f(\theta, \epsilon) = \log\frac{l(g(\theta,\epsilon,\lambda)|\theta, \lambda)}{p(g(\theta,\epsilon,\lambda)|\lambda)}$. 
    Then the ground-truth EIG can be represented as $U(\lambda) = \mathbb{E}[f(\theta, \epsilon)]$.
    Using Minkowski’s inequality, the mean squared error of $\widehat{U}_{srNMC}^{M}(\lambda)$ can be bounded by
    \begin{equation}
        \begin{aligned}
            \mathbb{E}[(U(\lambda)-\widehat{U}_{srNMC}^{M}(\lambda))^2]= 
            \lVert U(\lambda)-\widehat{U}_{srNMC}^{M}(\lambda) \rVert^2_2
            \leq U^2+V^2+2UV
            \leq 2(U^2+V^2),
        \end{aligned}
    \end{equation}
    where the expectation is taken over $\prod_{i=1}^M \pi_\mathrm{\theta}(\theta^{(i)})\pi_\mathrm{\epsilon}(\epsilon^{(i)})$  , $U = \Big\lVert U(\lambda)-\frac{1}{M}\sum_{i=1}^M f(\theta^{(i)}, \epsilon^{(i)})\Big\rVert_2$ and $V = \Big\lVert\frac{1}{M}\sum_{i=1}^M f(\theta^{(i)}, \epsilon^{(i)}) - \widehat{U}_{srNMC}^{M}(\lambda)\Big\rVert_2$. Since $l(g(\theta,\epsilon,\lambda)|\theta', \lambda)$ is uniformly bounded from below and above, we have $f\in L^2$. Also noting that $U$ is the square of mean error of a Monte Carlo estimation, it is easy to get $U=O(1/\sqrt{M})$. Now we turn to bound $V$. Using the assumption that $l(g(\theta,\epsilon,\lambda)|\theta', \lambda)\geq C_1$ a.s., we have 
    \begin{equation}
        \begin{aligned}
            V &= \Big\lVert \frac{1}{M}\sum_{i=1}^M \log \frac{\frac{1}{M}\sum_{j=1}^M l(y^{(i)}|\theta^{(j)}, \lambda)}{p(y^{(i)}|\lambda)} \Big\rVert_2 \\
            &\leq  \frac{1}{M}\sum_{i=1}^M\Big\lVert \log \frac{\frac{1}{M}\sum_{j=1}^M l(y^{(i)}|\theta^{(j)}, \lambda)}{p(y^{(i)}|\lambda)} \Big\rVert_2 \\
            & \leq \frac{1}{C_1 M}\sum_{i=1}^M\Big\lVert {\frac{1}{M}\sum_{j=1}^M l(y^{(i)}|\theta^{(j)}, \lambda)}-{p(y^{(i)}|\lambda)} \Big\rVert_2,
        \end{aligned}
    \end{equation}
    where $y^{(i)} = g(\theta^{(i)},\epsilon^{(i)},\lambda)$. For each term of the above equation, by Minkowski’s inequality we have
    \begin{equation}
        \begin{aligned}
            &\Big\lVert {\frac{1}{M}\sum_{j=1}^M l(y^{(i)}|\theta^{(j)}, \lambda)}-{p(y^{(i)}|\lambda)} \Big\rVert_2 
            \\
            \leq& \frac{1}{M}\Big\lVert l(y^{(i)}|\theta^{(i)}, \lambda)-l(y^{(i)}|\theta'^{(i)}, \lambda) \Big\rVert_2 + \Big\lVert {\frac{1}{M}\sum_{j\neq i}^M l(y^{(i)}|\theta^{(j)}, \lambda)}+\frac{1}{M}l(y^{(i)}|\theta'^{(i)}, \lambda)-{p(y^{(i)}|\lambda)} \Big\rVert_2,
        \end{aligned}
    \end{equation}
    where $\theta'^{(i)}\sim \pi_{\theta}(\theta)$. Using the assumption that $l(g(\theta,\epsilon,\lambda)|\theta', \lambda)\leq C_2$, the first term of above equation can be bounded by $2C_2/M$. 
    The square of the second equation can be bounded as 
    \begin{equation}
        \begin{aligned}
            &\Big\lVert {\frac{1}{M}\sum_{j\neq i}^M l(y^{(i)}|\theta^{(j)}, \lambda)}+\frac{1}{M}l(y^{(i)}|\theta'^{(i)}, \lambda)-{p(y^{(i)}|\lambda)} \Big\rVert_2^2
            \\
            =& \mathbb{E}\Big[\mathbb{E}\Big[ \Big({\frac{1}{M}\sum_{j\neq i}^M l(y^{(i)}|\theta^{(j)}, \lambda)}+\frac{1}{M}l(y^{(i)}|\theta'^{(i)}, \lambda)-{p(y^{(i)}|\lambda)}\Big)^2
            \Big|y^{(i)} \Big]\Big]\\
            =&\mathbb{E}\Big[\mathrm{Var}\Big[ {\frac{1}{M}\sum_{j\neq i}^M l(y^{(i)}|\theta^{(j)}, \lambda)}+\frac{1}{M}l(y^{(i)}|\theta'^{(i)}, \lambda)
            \Big|y^{(i)} \Big]\Big]\\
            =& \frac{1}{M}\mathbb{E}\Big[\mathrm{Var}\Big[l(y^{(i)}|\theta'^{(i)}, \lambda)
            \Big|y^{(i)} \Big]\Big]\\
        \end{aligned}
    \end{equation}
    where the first equality is obtained by the tower property of conditional expectation. It can be further bounded $O(1/M)$ noting that $l(y^{(i)}|\theta'^{(i)}, \lambda)$ is uniformly bounded from below and above. Therefore, we have $V=O(1/\sqrt{M})$ as well. Finally, using the obtained bounds of $U$ and $V$ we get the mean squared error of $\widehat{U}_{srNMC}^{M}(\lambda)$ 
    \begin{equation}
        \mathbb{E}[(U(\lambda)-\widehat{U}_{srNMC}^{M}(\lambda))^2]
        \leq 2(U^2+V^2)=O(1/M).
    \end{equation}
    
\end{proof}

\begin{theorem}\label{app:thm4}
    For any $C$ satisfying $0\leq C\leq U(\lambda)/2$, if $M\leq \exp(U(\lambda)/2)$, we have 
    \begin{equation}
        U(\lambda) - \widehat{U}_{srNMC}^{M}(\lambda)>C.
    \end{equation}
\end{theorem}
\begin{proof}
    We first show that $\widehat{U}_{srNMC}^{M}(\lambda)$ does not exceed $\log M$. By the definition of srNMC, we have 
    \begin{equation}
        \begin{aligned}
            \widehat{U}_{srNMC}^{M}(\lambda) &= \frac{1}{M}\sum_{i=1}^M \log\frac{l(y^{(i)}|\theta^{(i)}, \lambda)}{\frac{1}{M}\sum_{j=1}^M l(y^{(i)}|\theta^{(j)}, \lambda)}\\
            &\leq  \frac{1}{M}\sum_{i=1}^M \log\frac{l(y^{(i)}|\theta^{(i)}, \lambda)}{\frac{1}{M} l(y^{(i)}|\theta^{(i)}, \lambda)}\\
            &\leq \frac{1}{M}\sum_{i=1}^M \log M = \log M,
        \end{aligned}    
    \end{equation} 
    where $y^{(i)} = g(\theta^{(i)},\epsilon^{(i)},\lambda)$.
    Using this inequality and $M\leq \exp(U(\lambda)/2)$, it is easy to get 
    \begin{equation}
        \begin{aligned}
            U(\lambda) - \widehat{U}_{srNMC}^{M}(\lambda) &\geq U(\lambda) - \log M \\
            & \geq U(\lambda)/2 \geq C.
        \end{aligned}
    \end{equation}
\end{proof}

\section{Further details of experiments}
\subsection{EIG Gradient Estimation Accuracy}
We assume 3 design variables (i.e. $\lambda=(\lambda_1,\lambda_2,\lambda_3)$) for this test. The 20 independent designs are uniformly drawn from $[-1,1]^3$. To estimate the biases associated with the methods under investigation, we perform 100 independent trials. Table~\ref{maxent:tab1} summarizes the number of samples used to estimate the gradient for a single design.

\begin{table}[ht]
\centering
\begin{tabular}{|l|l|}
\hline
Method & Number of samples \\
\hline
BEEG-AP & $100\times 100(M)$ \\
UEEG-MCMC & $100\times 100(L)$ \\
PCE & $100\times 100(M)\times (100(N)+1)$ \\
\hline
\end{tabular}
\caption{Method and number of samples for the test of EIG gradient estimation accuracy.}
\label{maxent:tab1}
\end{table}

\subsection{A Toy Algebraic Model}\label{maxent:app_toy}
In this test, we allocate a simulation budget of $2\times 10^4$ for optimizing the design variables for each method. For UEEG-MCMC, we use slice sampling \cite{neal2003slice} with a thinning factor of 2 to draw 10 samples for the posterior sampling in Eq.~\eqref{maxent:eq4}. For ACE, we use a conditional Gaussian posterior inference network $q_\phi(\cdot|y) = \mathrm{normpdf}(\cdot|\mu_{\phi}(y), e^{2\sigma_{\phi}(y)})$, where $\mu_{\phi}$ and $\sigma_{\phi}$ are the two outputs of a two-layer fully connected network with 50 hidden units and ReLU activation functions. For GradBED, we use a two-layer fully connected network $T_{\psi}$ with 50 hidden units and ReLU activation functions. The number of samples used to estimate a single gradient for each method is given by Table~\ref{maxent:tab2}.

\begin{table}[ht]
\centering
\begin{tabular}{|l|l|}
\hline
Method & Number of samples \\
\hline
BEEG-AP & $100(M)$ \\
UEEG-MCMC & $10(M)\times \mathrm{No.~samples~from~slice~sampling}$ \\
ACE & $10(M)\times (10(N)+1)$ \\
PCE & $10(M)\times (10(N)+1)$ \\
GradBED & $100(M)$\\
\hline
\end{tabular}
\caption{Method and number of samples for the toy model.}
\label{maxent:tab2}
\end{table}

To validate the quality of designs, we use NMC with 10,000 samples for each estimate of EIG in Fig.~\ref{fig:toy_val_large}. The estimated posterior entropy in Fig.~\ref{fig:toy_val_small} is obtained as follows. In our experimental setup, each method yields 20 final designs. For each of these designs, we conduct simulations resulting in 500 observed data points derived from the marginal likelihood, thus forming 500 \textit{guess} posteriors. Subsequently, kernel density estimation (KDE) is applied to these posteriors, utilizing 100 posterior samples for each, to approximate the entropies of these posteriors. The posterior entropy for each method is then computed by averaging these approximated entropies.

\subsection{PK Model}
In this application, we allocate a simulation budget of $2\times 10^6$ for optimizing the design variables for each method. For UEEG-MCMC, we use an adaptive Metropolis-Hastings (MH) method with a thinning factor of 95 to draw only one sample for the posterior sampling in Eq.~\eqref{maxent:eq4}. For ACE, we utilize the Mixture Density Network \cite{bishop1994mixture} as our chosen posterior inference network, with 3 hidden layers, 50 hidden units in each layer and ReLU activation function. For GradBED, we follow the network settings in \cite{kleinegesse2020bayesian}. Specifically, we use a one-layer connected network $T_{\psi}$ consisting of 300 hidden units and ReLU activation functions. The number of samples used to estimate a single gradient for each method is given by Table~\ref{maxent:tab2}.

\begin{table}[ht]
\centering
\begin{tabular}{|l|l|}
\hline
Method & Number of samples \\
\hline
BEEG-AP & $100(M)$ \\
UEEG-MCMC & $1(M)\times 100(\mathrm{approximated~cost~of~the~adaptive~MH}$) \\
ACE & $10(M)\times (10(N)+1)$ \\
PCE & $10(M)\times (10(N)+1)$ \\
GradBED & $100(M)$\\
\hline
\end{tabular}
\caption{Method and number of samples for PK model and STAT5 model.}
\label{maxent:tab2}
\end{table}

In the validation experiments, we use NMC with 10,000 samples for each estimate of EIG in Fig.~\ref{fig:pk_gt}. In Fig.~\ref{fig:pk_val}, the posterior entropy for each method is estimated via 1000 independent trials (see Appendix~\ref{maxent:app_toy} for the procedure). For each trial, we use KDE to estimate the \textit{guess} posterior entropy with 1000 samples.

\subsection{STAT5 Model}
The dynamics of STAT5 populations can be described by four coupled ODEs
\begin{equation}
\begin{aligned}
& \dot{x}_1=-k_1 x_1 \operatorname{EpoR}_A(t)+k_2 x_3(t-\tau) \\
& \dot{x}_2=-x_2^2+k_1 x_1 \operatorname{EpoR}_A(t) \\
& \dot{x}_3=-k_2 x_3+x_2^2 \\
& \dot{x}_4=-k_2 x_3(t-\tau)+k_2 x_3.
\end{aligned}
\end{equation}
The variables in the model are defined as follows. $x_1$ represents unphosphorylated STAT5. $x_2$ and $x_3$ represent tyrosine phosphorylated monomeric STAT5 and tyrosine phosphorylated dimeric STAT5 respectively. $x_4$ is the nuclear STAT5. $\operatorname{EpoR}_A(t)$ describes the erythropoietin receptor activity that determines the STAT5 response.

We assume that $x_1(0)=3.71$ is the only non-zero initial state of the ODES. To facilitate the solvability of the ordinary differential equations (ODEs), as in \cite{peifer2007parameter, banga2008parameter}, we apply linear interpolation to synthesize the function $\operatorname{EpoR}_A(t)$ with the original data in \cite{swameye2003identification}, and use a delay chain of length $N$ to approximate the delayed term $x_3(t-\tau)$,
\begin{equation}
\begin{gathered}
\dot{q}_1=\frac{N}{\tau}\left(\operatorname{in}(t)-q_1\right) \\
\dot{q}_2=\frac{N}{\tau}\left(q_1-q_2\right) \\
\ldots \\
\dot{q}_{N-1}=\frac{N}{\tau}\left(q_{N-2}-q_{N-1}\right) \\
\text { out }=\frac{N}{\tau}\left(q_{N-1}-\operatorname{out}(t)\right),
\end{gathered}
\end{equation}
where $N=8$, $\operatorname{in}(t)=x_3(t)$ and $\operatorname{out}(t)=x_3(t-\tau)$.

To build the sampling path that supports backpropagation through ODE solutions, we utilize the package \texttt{torchdiffeq} \cite{chen2018neuralode} to solve the ODEs with 3/8-Runge-Kutta method \cite{butcher1996history}. Linear interpolation is then applied to get the observations at any measurement times.

In this application, we allocate a simulation budget of $5\times 10^5$ for optimizing the design variables for each method. The settings for the methods involved are consistent with those employed for the PK model, and the sample size for each method is also provided in Table~\ref{maxent:tab2}. In the validation experiments, the posterior entropy for each method shown in Fig.~\ref{fig:stats5_val_0} and Fig.~\ref{fig:stats5_val_1} is estimated via 100 independent trials (see Appendix~\ref{maxent:app_toy} for the procedure). For each trial, KDE is employed with 1000 posterior samples.

\begin{figure}
    \centering
    \includegraphics[width=0.49\textwidth]{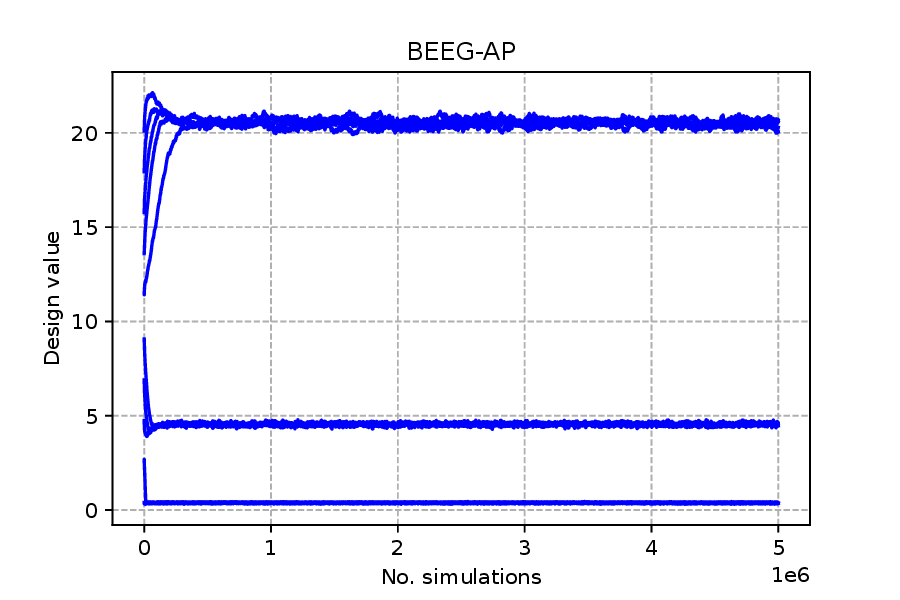}
    \includegraphics[width=0.49\textwidth]{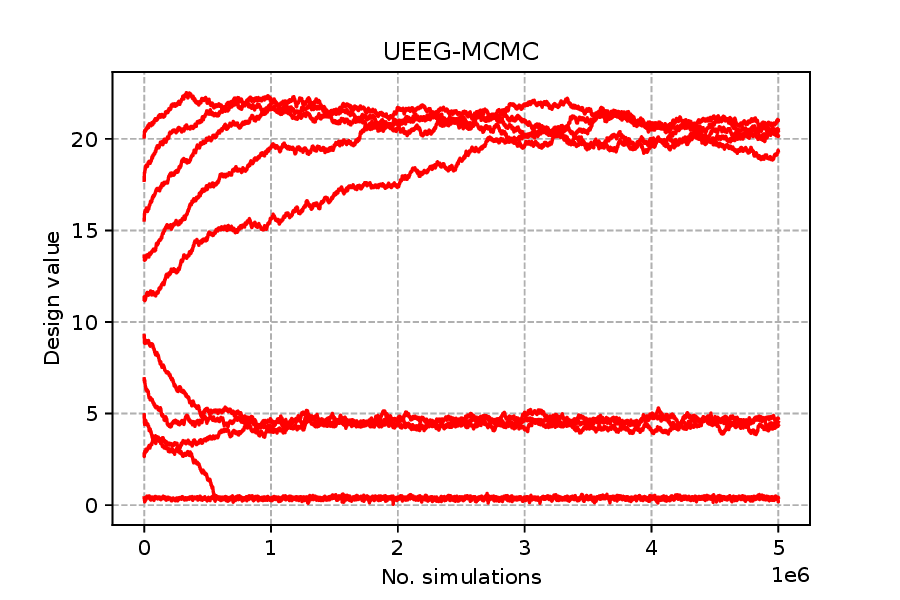}
    \includegraphics[width=0.49\textwidth]{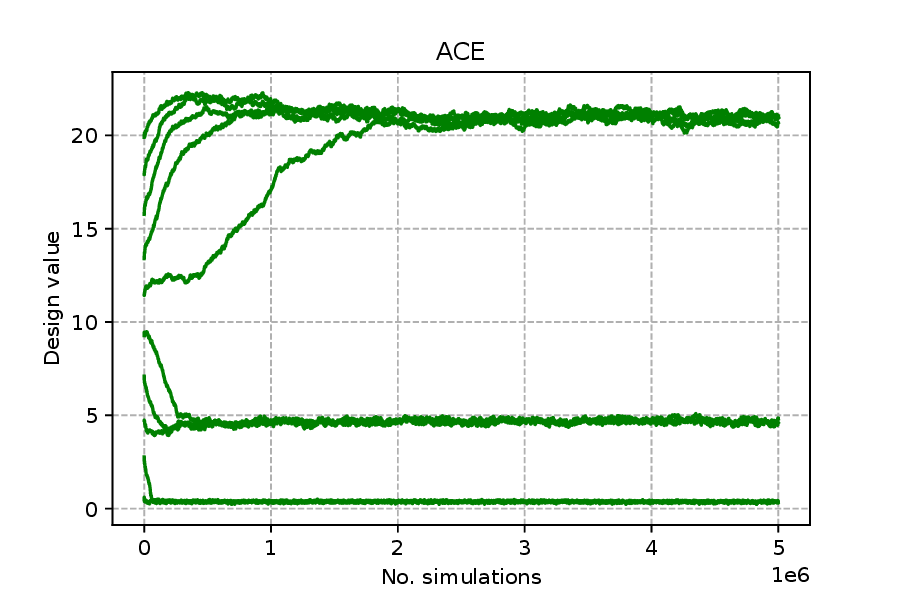}
    \includegraphics[width=0.49\textwidth]{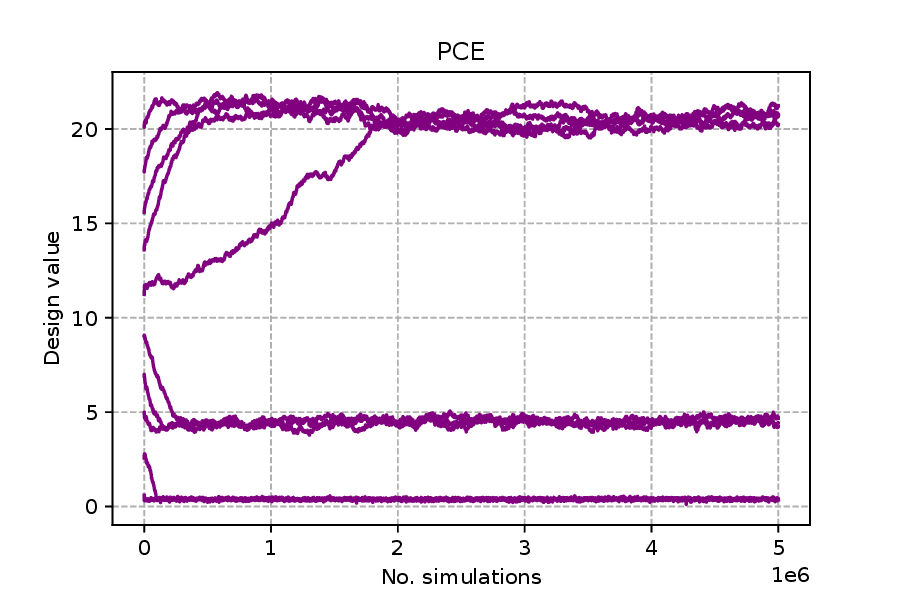}
    \includegraphics[width=0.49\textwidth]{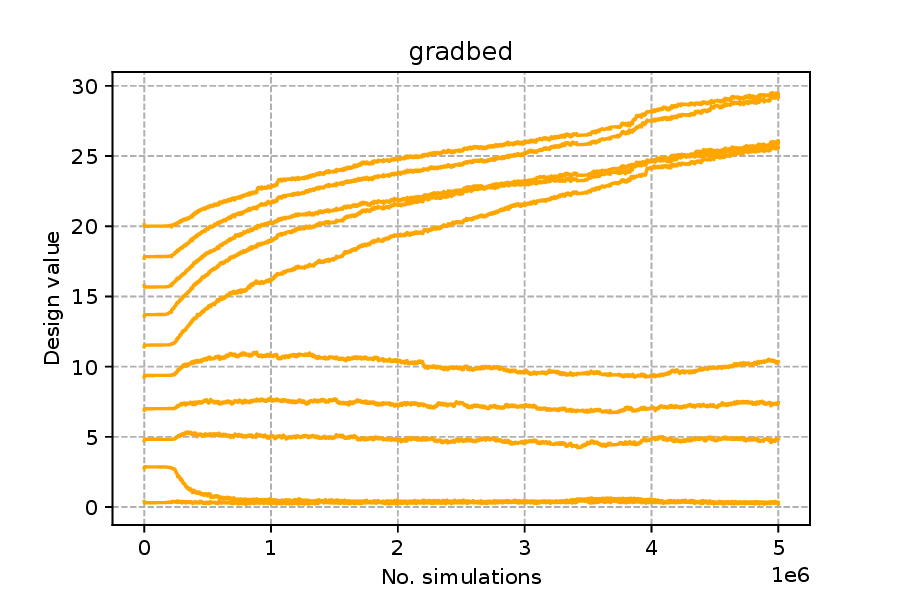}
    \caption{Convergence of the individual design dimensions for PK model with multiplicative noise $\mathcal{N}(0, 0.01)$ and additive noise $\mathcal{N}(0, 0.1)$.}
    \label{fig:pk_convergence_mix}
\end{figure}

\begin{figure}
    \centering
    \includegraphics[width=0.49\textwidth]{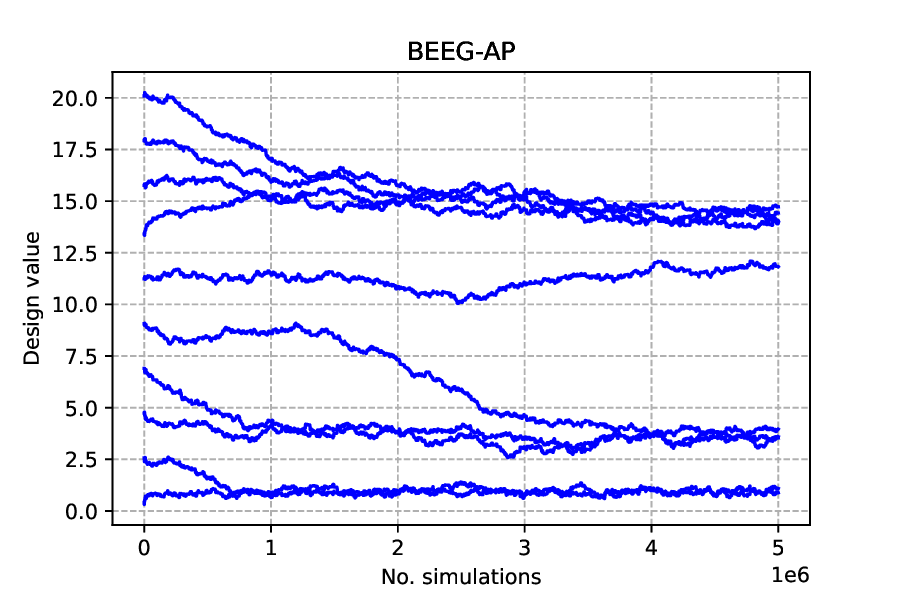}
    \includegraphics[width=0.49\textwidth]{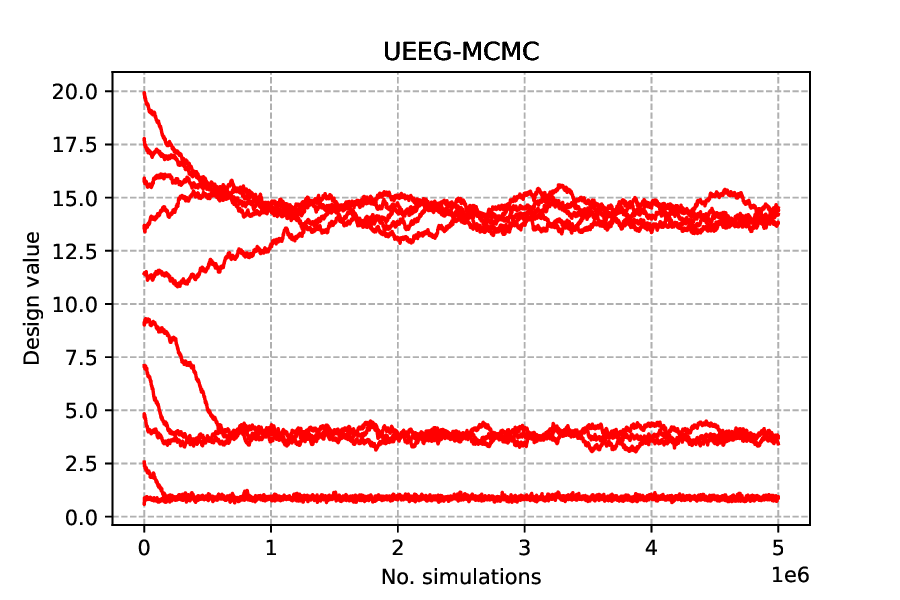}
    \includegraphics[width=0.49\textwidth]{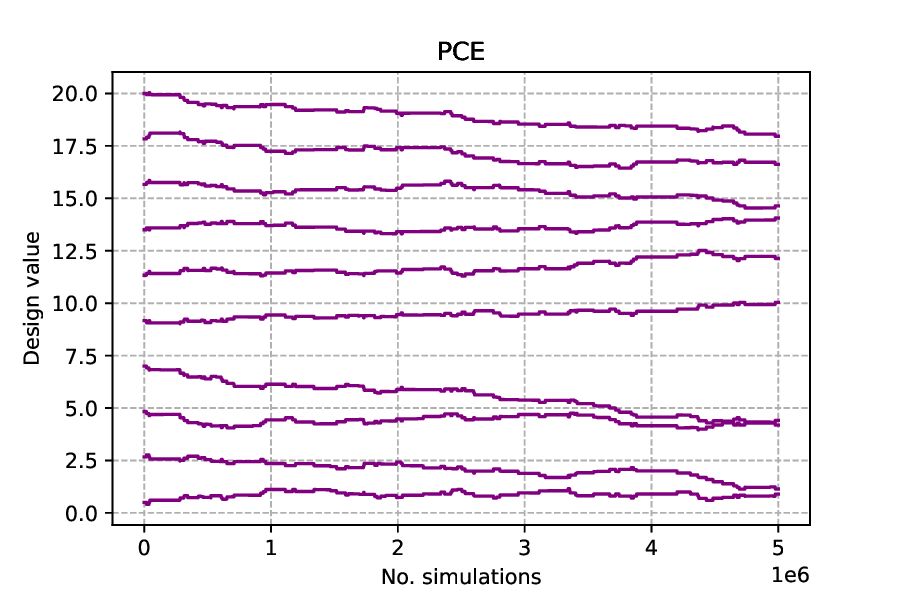}
    \includegraphics[width=0.49\textwidth]{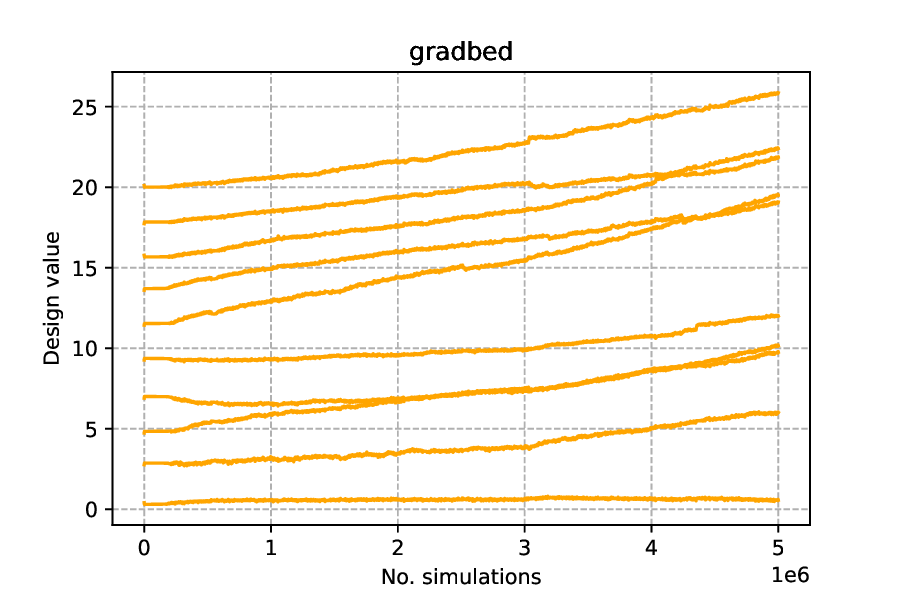}
    \caption{Convergence of the individual design dimensions for PK model with additive noise $\mathcal{N}(0, 0.001)$.}
    \label{fig:pk_convergence_add}
\end{figure}

\begin{figure}
    \centering
    \includegraphics[width=0.49\textwidth]{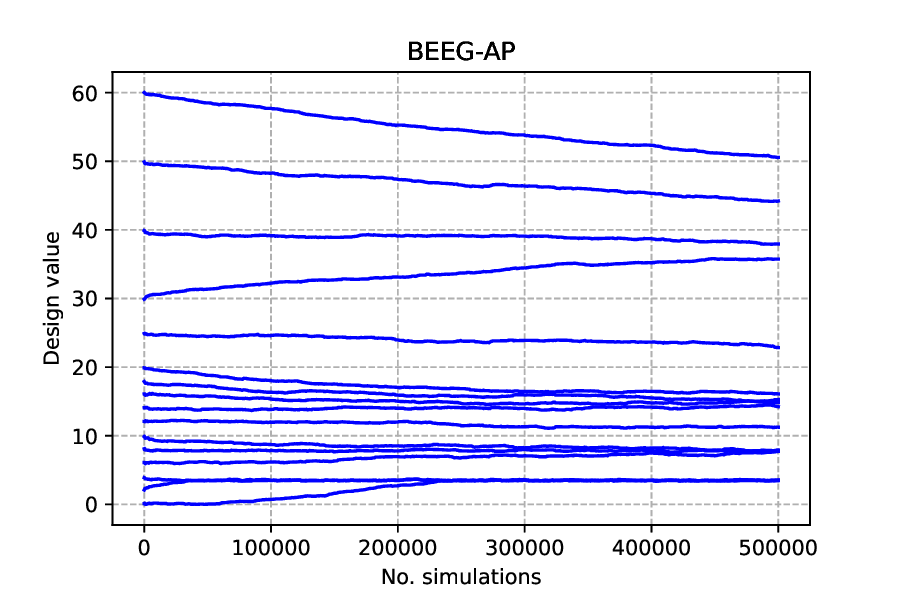}
    \includegraphics[width=0.49\textwidth]{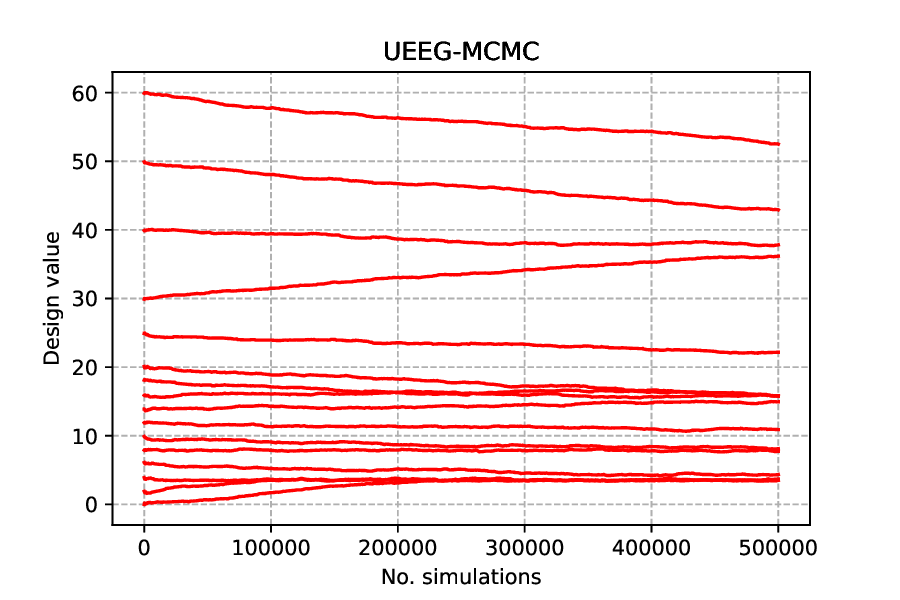}
    \includegraphics[width=0.49\textwidth]{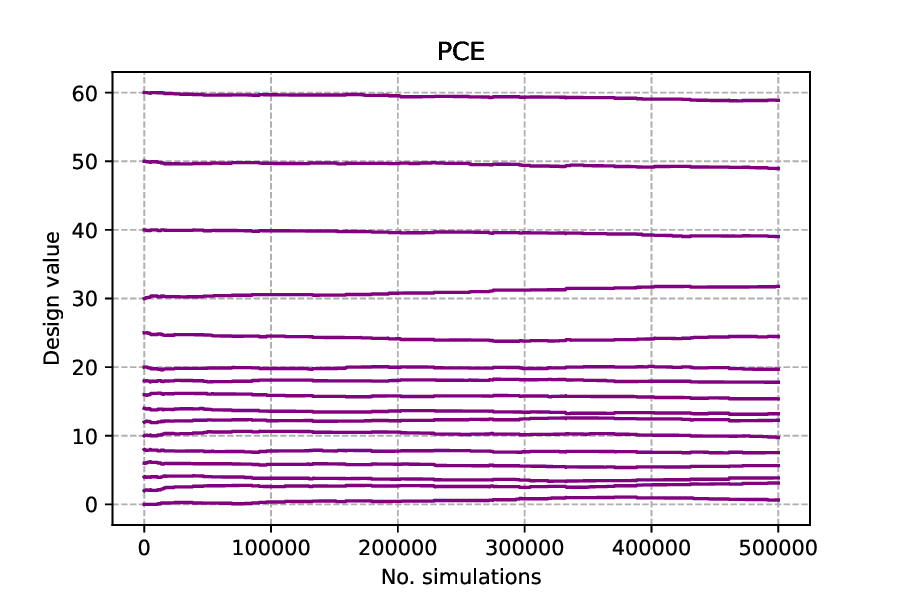}
    \includegraphics[width=0.49\textwidth]{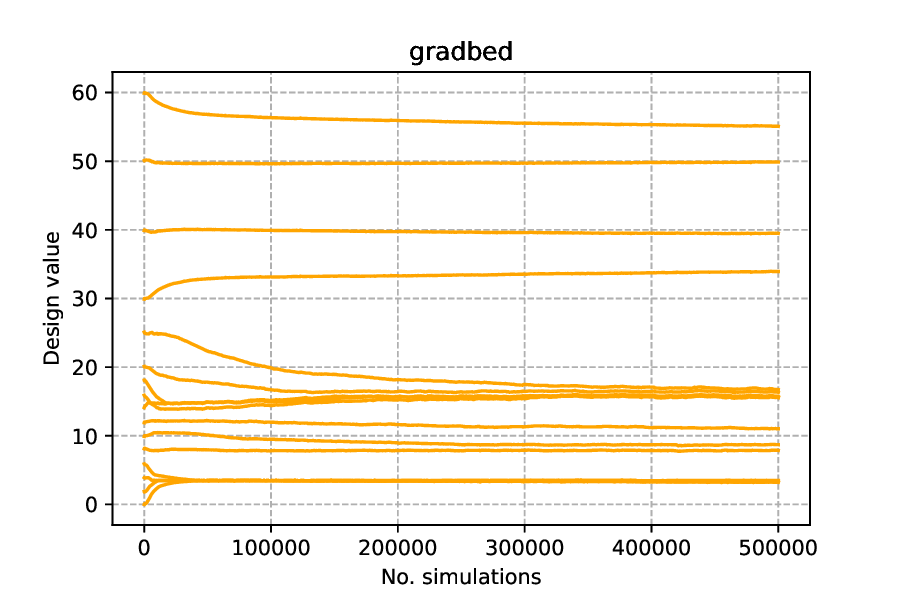}
    \caption{Convergence of the individual design dimensions for STAT5 model with additive noise $\mathcal{N}(0, 0.01)$.}
    \label{fig:stats5_convergence_add_0}
\end{figure}

\begin{figure}
    \centering
    \includegraphics[width=0.49\textwidth]{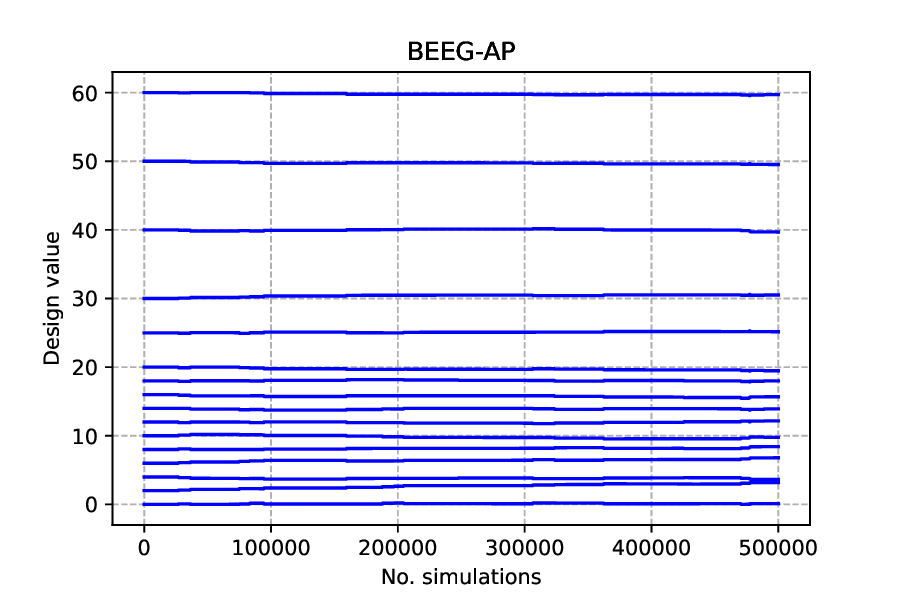}
    \includegraphics[width=0.49\textwidth]{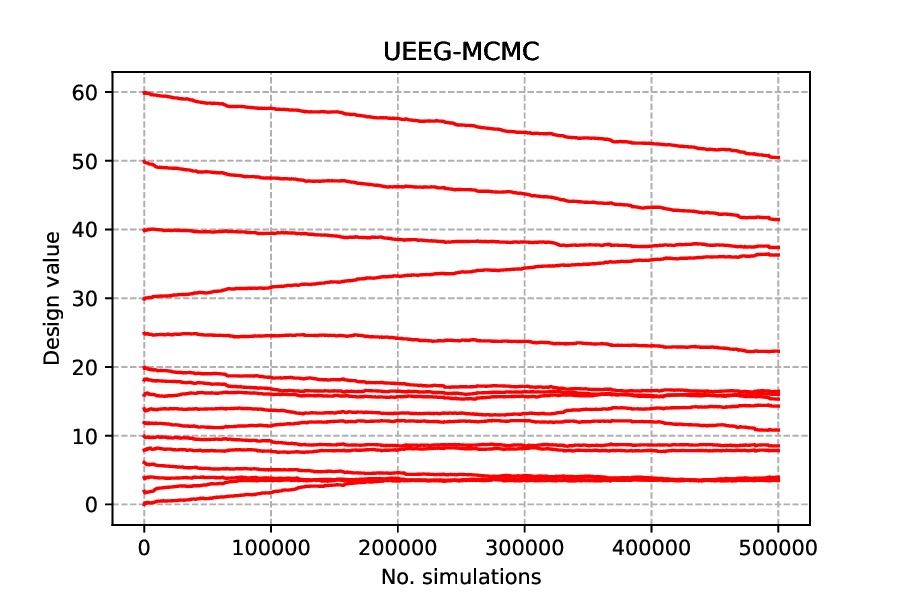}
    \includegraphics[width=0.49\textwidth]{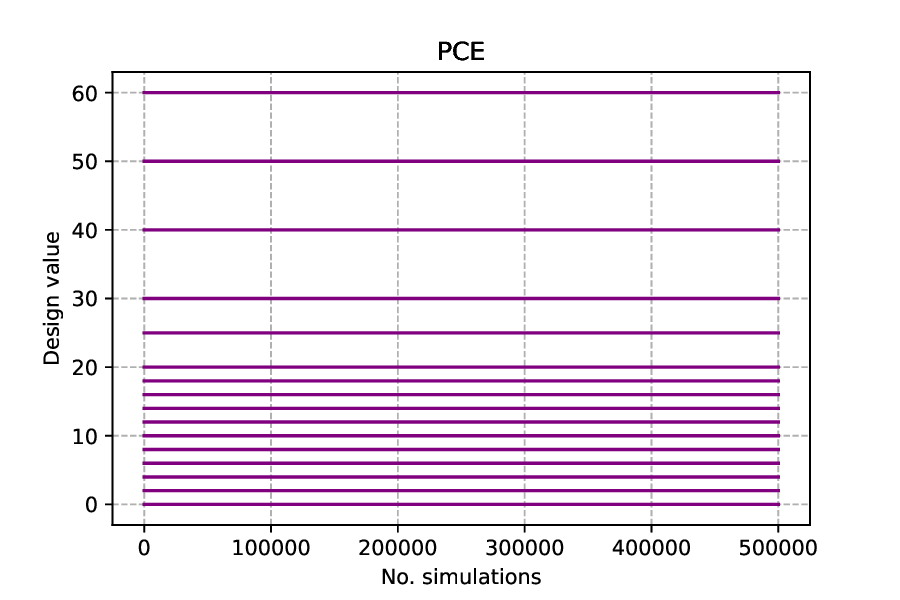}
    \includegraphics[width=0.49\textwidth]{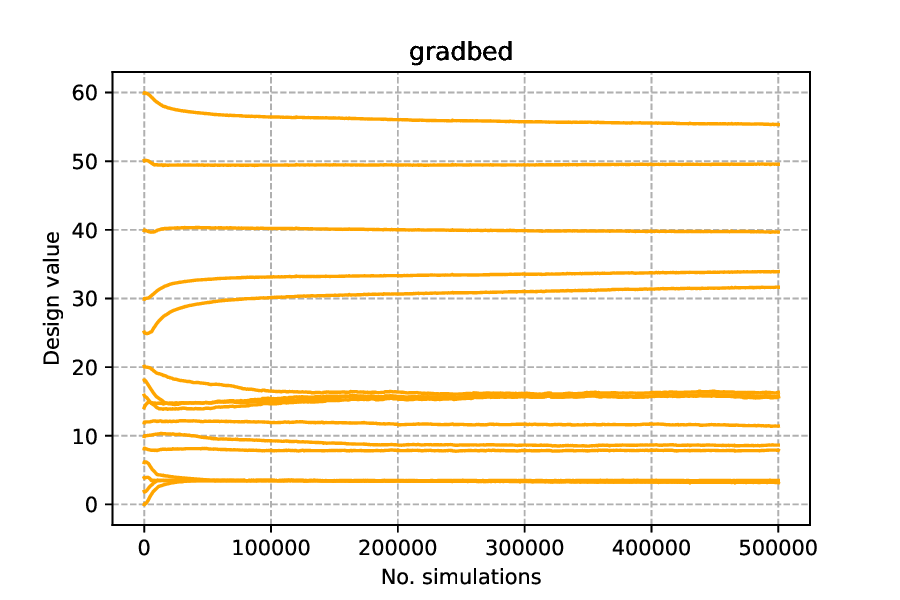}
    \caption{Convergence of the individual design dimensions for STAT5 model with additive noise $\mathcal{N}(0, 0.001)$.}
    \label{fig:stats5_convergence_add_1}
\end{figure}
\fi

\end{document}